 \documentclass[12pt]{alt2021} % Include author names

\title[Online $k$-means Clustering]{Unexpected Effects of Online no-Substitution $k$-means Clustering}
\usepackage{times}

\newtheorem{claim}[theorem]{Claim}

\usepackage{nicefrac}       % compact symbols for 1/2, etc.
\usepackage{microtype}      % microtypography

\usepackage[export]{adjustbox}
\usepackage{wrapfig}

\usepackage{xcolor}
\definecolor{titleColor}{RGB}{0,22,166}

\newcommand{\norm}[1]{\left\lVert#1\right\rVert}
\newcommand{\inner}[2]{\langle#1,#2\rangle}
\newcommand{\dis}[2]{d(#1,#2)}
\DeclareMathOperator*{\argmin}{arg\,min}
\DeclareMathOperator*{\argmax}{arg\,max}
\DeclareMathOperator{\E}{\mathbb{E}}

\newcommand{\floor}[1]{\left \lfloor #1\right \rfloor}
\newcommand{\reals}{\mathbb{R}}

\newcommand{\poly}{\texttt{poly}}

\usepackage{algorithmic}
\usepackage{graphicx}
\usepackage{algorithm}

\usepackage{eqparbox}

\usepackage{etoolbox}

\altauthor{%
 \Name{Michal Moshkovitz} \Email{mmoshkovitz@eng.ucsd.edu}\\
 \addr University of California, San Diego
}

\begin{document}

\maketitle

\begin{abstract}
Offline $k$-means clustering was studied extensively, and algorithms with a constant approximation are available.
However, online clustering is still uncharted. New factors come into play: the ordering of the dataset and whether the number of points, $n$, is known in advance or not. Their exact effects are unknown.  
In this paper we focus on the online setting where the decisions are irreversible: after a point arrives, the algorithm needs to decide whether to take the point as a center or not, and this decision is final.  
How many centers are needed and sufficient to achieve constant approximation in this setting? 
We show upper and lower bounds for all the different cases. These bounds are exactly the same up to a constant, thus achieving optimal bounds. 
For example, for $k$-means cost with constant $k>1$ and random order, $\Theta(\log n)$ centers are enough to achieve a constant approximation, while the mere a priori knowledge of $n$ reduces the number of centers to a constant.
These bounds hold for any distance function that obeys a triangle-type inequality.
\end{abstract}


\section{Introduction}
Clustering is an unsupervised learning problem where the goal is to group data into a few clusters. %similar sets.
It is an important exploratory data analysis step used in various domains like bioinformatics, image analysis, and information retrieval. 
In the literature, there are many algorithms for clustering in the \emph{offline setting}, where all the points in the dataset are given in advance (\cite{kanungo02,arthur07,aggarwal09,ahmadian19}).
%In this setting, the entire data is given at once as an input. %, and the desired number of clusters $k$ is given as a parameter. 
The output of a center-based clustering algorithm is a set of \emph{centers}  %$c_1,\ldots, c_{f(k)}$ 
in the dataset, where each center is a ``representative'' of one cluster. %Specifically, the size of the data is known 

In the \emph{online no-substitution setting}, points in the dataset arrive one after another, and a decision whether to take the current point as a center needs to be made before observing the next point.
As a motivating example, \cite{hess19} suggested a clinical trial of a new drug, where patients are the points and patients given the new drug are the centers. The goal is to provide the new drug to the smallest number of patients, to avoid unnecessary risk (minimize the number of centers), while ensuring a good representation of the entire population in the trial (small $k$-means cost). Once a patient is out of the clinic, she cannot be tested, and after she took the drug, she cannot undo it --- thus, decisions are irreversible. 
Studying the online setting is more important these days as new data is constantly generated. (\cite{l2017machine,marx2013biology}). 
In the online setting, new factors come into place: the order of the input points (random or worst order) and whether the number of points in the dataset is known in advance or not. 
%These factors are meaningless in the offline setting, because the entire data is given in advance.   
It makes sense that the order of the points would impact performance, but by how much? Can a priori knowledge of the size of the dataset improve performance? In this paper we answer these questions and find that these two new factors have unexpected effects on online clustering. Specifically, we show that the ordering of the dataset can exponentially increase the number of centers. We also prove that merely knowing the size of the dataset can reduce the number of centers logarithmically. 

%In this paper we find the optimal performance with respect to these new parameters.   We do so by providing new algorithms and a matching lower bounds. 

%Learning is mostly done in the offline setting. 
%Lately the online setting become more important. 
%Clustering is important problem in many settings. 
%What is the different between online and offline in the clustering case?  
%
%* Importance of clustering and online clustering 
%
%* Works on offline learning: upper bounds (algorithms); works on lower bound 
%
%* The model + motivation 
%
%* Our results 
%
%* Related work: * Add somewhere: prior algorithm is not optimal as it must depend on the distance 
%
%
%\subsection{Motivation}
%We are considering a framework for online clustering where a learner receives examples one by one and as each example appears it is chosen whether to take it as a center or not and the decision cannot be reversed. 
%\begin{enumerate}
%\item Why is the framework interesting? 
%\begin{itemize}
%\item k-means ++ uses this framework 
%\item previous result on online k mean uses this framework 
%\item present bounds to looser models 
%\end{itemize}
%\item We improve previous results on online k-means 
%\item Theory versus practice: we show optimal theoretical results and it's also a simple (easy to implement) algorithm that is empirically good 
%\item Sheds light on the question of online versus offline learning
%%\item Related work: other lower bounds
%\end{enumerate}

\subsection{The online framework}
%To ease the presentation we focus on the \emph{k-means cost}. Our results apply also to more general cost functions, as discussed in Section~\ref{sec:general_cost_function}.
To ease the presentation we focus on the $k$-means cost, though our results apply to more general cost functions, as discussed in Section~\ref{sec:general_cost_function}.
For any dataset $D=\{x_1,\ldots,x_n\}\subseteq\mathbb{R}^d$ and desired number of clusters $k$, the \emph{$k$-means cost} is defined as the sum of squared $\ell_2$-distances of each point in the dataset to its closest center:
\begin{eqnarray}\label{eq:k_menas}
cost(c_1,\ldots, c_k) = \sum_{t=1}^n \norm{x_t-c(x_t)}^2,
\end{eqnarray}
where $c(x)$ is the closest center to $x$, i.e.,  $c(x)=\argmin_{c_i\in \{c_1,\ldots, c_k\}} \norm{x-c_i}.$ We denote  by $cost(opt_k)$ the optimal cost using $k$ centers\footnote{More generally, one can ease the requirement, and allow the centers to be in $\reals^d$ and not necessarily in $D.$ This can improve $cost(opt_k)$ only by a factor of $2$, see Lemma~\ref{lemma:random_point_in_cluster}.}:  $cost(opt_k):=\min_{c_1,\ldots, c_k\in D}cost(c_1,\ldots,c_k).$
In the \emph{offline setting} 
%\emph{k-means clustering} 
an algorithm receives a dataset $D$
%$n$ points $D=\{x_1,\ldots,x_n\}\subseteq\mathbb{R}^d$ 
and a desired number of clusters $k$, and in $\textit{poly}(n)$ time returns a set of centers $c_1,\ldots, c_\ell \in D$ such that (1) the number of centers, $\ell$, is close as possible
to $k$ and (2) $cost(c_1,\ldots,c_\ell)$ is close as possible to $cost(opt_k)$. %the the optimal clustering with $k$ centers.  In the offline setting the goal is to design an efficient algorithm. 

% In \emph{k-means clustering}, the goal is to find a set of centers $c_1,\ldots,c_\ell\in\reals^d$ with (1) a small number $\ell$ of centers, and (2) a small \emph{$k$-means cost} $cost(c_1,\ldots, c_\ell)$, which is the sum of distances of each point in the dataset to its closest center. In (2), ``small" means compared to $cost(opt_k)$ which is the optimal offline cost with just $k$ centers. A clustering algorithm should be efficient, i.e., its running time should be polynomial with respect to the size of the dataset. 

%In \emph{k-means clustering} the goal is to find a set of centers $c_1,\ldots,c_k\in\reals^d$ that minimizes (1) number of centers, $k$ (2) k-means cost $cost(c_1,\ldots, c_k)$, which is the sum of distances of each point in the dataset to its closest center.  
%In the big data era that we are in, we cannot except to defer all the computations to the end. 
We focus on the following \emph{online $k$-means} setting:
%Similar to previous works \cite{liberty16}, we consider the case 
At each time step, when a new point in the dataset arrives, the algorithm needs to decide whether to take it as a center or not. The decisions cannot be changed after the next point arrives. Points that were not chosen as centers cannot be considered as centers later on, and points that were chosen as centers cannot be removed from the set of centers.   
This setting was used and motivated in \cite{hess19}.  %\cite{liberty16,hess19}. %In \cite{hess19}, this setting was called the \emph{no-substitution} setting.

In this paper, we consider constant approximation algorithms, meaning $cost(c_1,\ldots,c_\ell)\leq a\cdot cost(opt_k)$, where $a$ is some constant. 
%$cost(alg)$ is the cost induced by the centers returned by the algorithm 
%and $cost(opt_k)$ 
We refer to such a clustering algorithm as a $\Theta(1)$-approximation. 
 The goal of the online algorithm is to minimize the number of centers $\ell$ and make it as close as possible to $k$.
 
%  To summarize, an online learning algorithm is 
%  \begin{itemize}
%      \item Given: parameters $a,k$ (and maybe $n$), and then the dataset $x_1,\ldots, x_n$ one by one  
%      \item Online outputs: centers $c_1,\ldots,c_\ell\in D$ with $cost(c_1,\ldots,c_\ell)\leq a\cdot cost(opt_k)$ and $\ell$ is minimized
%  \end{itemize}
 
 %Therefore, an online clustering algorithm efficiently solves the following optimization problem:
% \begin{equation*}
% \begin{array}{ll@{}ll}
% \text{minimize}  & \displaystyle&\ell &\\
% \text{subject to}& \displaystyle& cost(c_1,\ldots, c_\ell)\leq a\cdot cost(opt_k)  &
% \end{array}
% \end{equation*}
%optimal cost case, where the cost of the algorithm is at most a constant from the optimal cost.
%We explore what is the optimal number of centers. More formal definition of the problem appears in Section~\ref{sec:on_line_model}. 
%%%%% Suggestion: formally define the online setting

%\subsection{Results summary}
\subsection{Our contribution}
\paragraph{Identifying principal factors.} Many factors might affect the quality of online $k$-means algorithms: the order of the points (random or worst order), dimension size, number of points, and number of clusters.
A conceptual contribution of this paper is the observation that the new factors in online clustering (namely, the ordering and whether the number of points is known in advance) significantly influence the optimal algorithms' performance.
 The dimension, however, plays no role. If the order is arbitrary and $k>1$, the a priori knowledge of $n$ is irrelevant too. 

\begin{figure*}[t!]
\centering
    % \begin{subfigure}[t]{0.4\textwidth}
     \begin{subfigure}[]
    \newline%[The Offline Landscape]%[t]{0.3mm}
         \centering
         \includegraphics[width=.1\textwidth,valign=t]{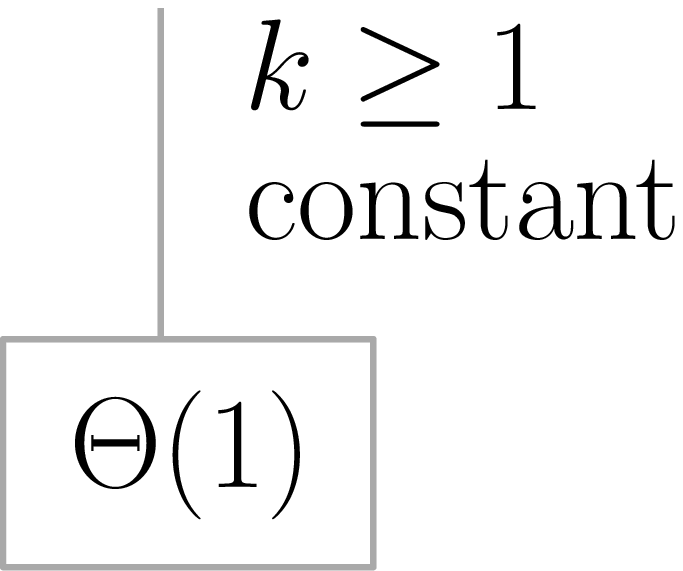}
         %\caption{The Offline Landscape}
         \label{fig:summary_results_offline}
     \end{subfigure}
     \hfill
   \begin{subfigure}[]%[The Online Landscape]%[t]{50mm}
 %    \begin{subfigure}[t]{0.6\textwidth}
          \centering
    % \caption{The Online Landscape}
      \includegraphics[width=.4\textwidth,valign=t]{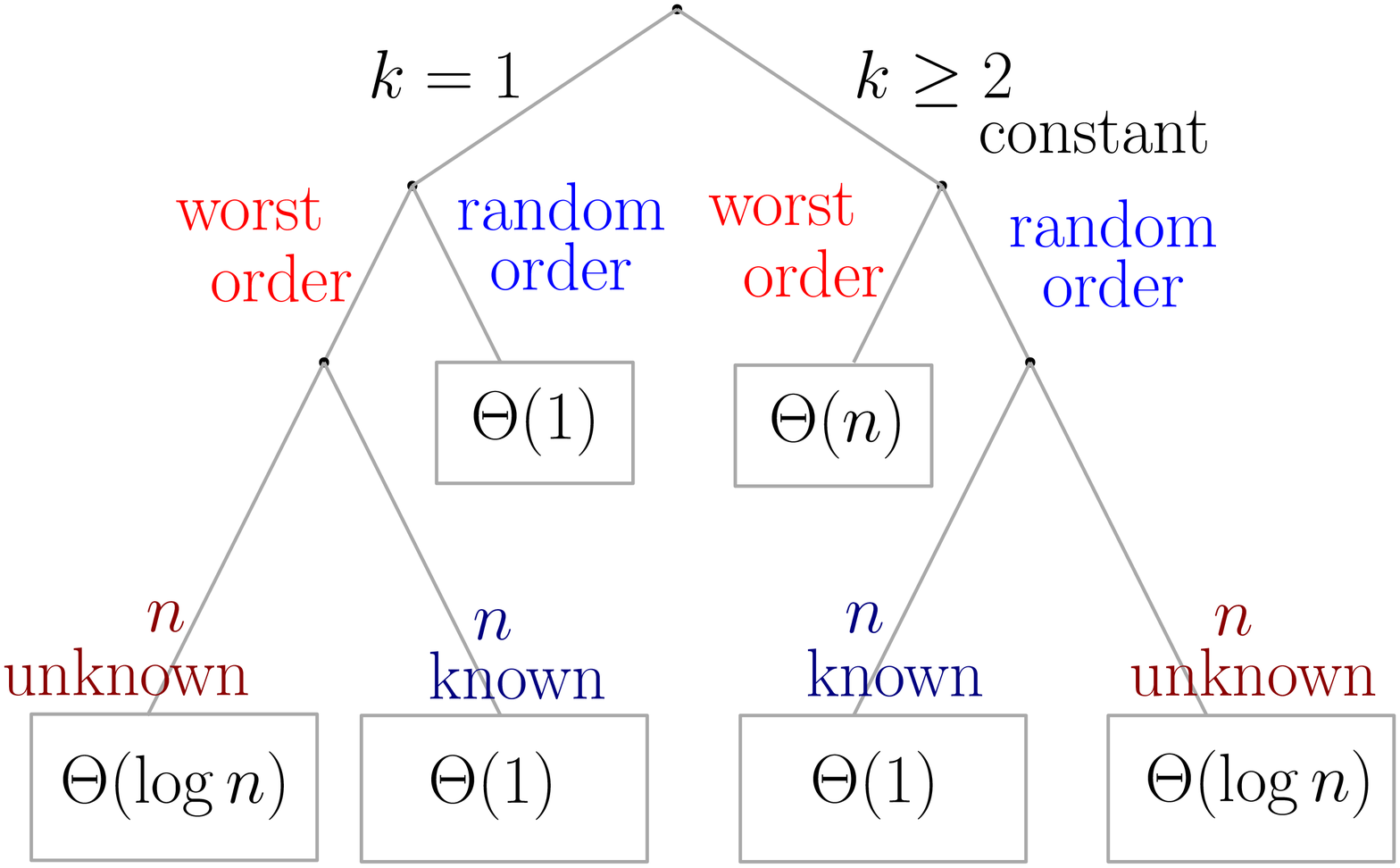}
      \label{fig:summary_results_online}
 \end{subfigure}
 \caption{Comparison between the $(a)$ offline and $(b)$ online settings. All algorithms are $\Theta(1)$-approximation. %(Defintion~\ref{dfn:c-approximation}).
 In rectangles: number of centers (these are optimal) (i) For constant $k$ only constant number of centers are needed in the offline setting (ii) In the online setting this paper uncovers a more complex behavior, see the text for details.} %All the algorithms described in the paper has polynomial running time with respect to the size of the data.}
 \label{fig:summary_results}
 \end{figure*}

% \begin{figure*}[t!]
%     \centering
%     \begin{minipage}[t]{.4\textwidth}
%         \centering
%      \subcaption{The Offline Landscape}
%      \includegraphics[width=.33\textwidth]{offline_scheme}
%      \label{fig:summary_results_offline}
%      \end{minipage}%
%      \begin{minipage}[t]{.6\textwidth}
%          \centering
%     \subcaption{The Online Landscape}
%      \includegraphics[width=.8\textwidth]{result_summary}
%      \label{fig:summary_results_online}
%      \end{minipage}
% \caption{Comparison between the offline and online settings. All algorithms are $\Theta(1)$-approximation. %(Defintion~\ref{dfn:c-approximation}).
% In rectangles: number of centers (these are optimal) (a) For constant $k$ only constant number of centers are needed in the offline setting (b) In the online setting this paper uncovers a more complex behavior, see the text for details.} %All the algorithms described in the paper has polynomial running time with respect to the size of the data.}
% \label{fig:summary_results}
% \end{figure*}

\paragraph{Entire landscape and optimal bounds.}
Henceforth we focus on the case of constant $k$. 
In the offline setting, an efficient algorithm that returns  $\Theta(1)$ centers is known (see \cite{arthur07,aggarwal09}). % for constant $k.$ 
%are sufficient (and needed) for clustering with constant $k$. 
The case of $k=1$ is more straightforward than that: the optimal center is the average point\footnote{If the average point is not in the dataset, take the closest point to it.}. 
That is the end of the story for offline clustering.  However, in the online setting, the story only begins. In this paper, we pinpoint the exact number of centers needed and sufficient to achieve a constant approximation for different values of the new factors.

The online landscape that we map is summarized in Figure~\ref{fig:summary_results}, explained in detail in Sections~\ref{sec:k_1} and \ref{sec:bounds_k_2}, and listed next. (i) For $k=1$: if either the order is random or $n$ is known in advance, then simple algorithms show that $\Theta(1)$ centers are needed and sufficient to achieve a constant approximation. If the order is worst case and $n$ is unknown in advance, $\Theta(\log n)$ centers are needed and sufficient. (ii) For constant $k\geq 2$: if the order is arbitrary, then $\Theta(n)$ centers are needed (and obviously sufficient). If the order is random and $n$ is unknown in advance, then $\Theta(\log n)$ centers are needed and sufficient, but if $n$ is known in advance, then $\Theta(1)$ centers are needed and sufficient.

\paragraph{Universality.} Interestingly, the landscape we have described is the same for any cost with a distance function that obeys a triangle-type inequality (e.g., $k$-medians, or more generally $\ell_p$ norms with constant $p$). This is proved in Section~\ref{sec:general_cost_function}.

\paragraph{Technical contribution.} One of the main technical contributions are two new algorithms for the case that the points arrive in random order and $k\geq 2$. One algorithm is for the case that $n$ is known in advance and thus, the algorithm can observe a small fraction of the data without taking any of these points as centers. Since the order is random, these points represent the entire data, thus it is becoming easier to choose which points to take as centers. The second algorithm is for the case that $n$ is unknown in advance and it uses farthest-first-traversal to find points that should be taken as centers. These algorithms are described in Section~\ref{sec:bounds_k_2}.

%In case that $k$ is not constant, the aclgorithms described in the paper function correctly but are not guaranteed to have the optimal number of centers.
% This paper highlights the importance of the new factors: 
% %parameters, 
% (1) knowing $n$ and (2) the order of the points, that both are meaningless in the offline setting. 
% %Thus, this paper sheds light on the question of the difficulty of online versus offline clustering.
% Thus, this paper proves that online clustering is strictly harder than offline k-means clustering. 

%\subsection{Paper Organization}
%In Section~\ref{sec:on_line_model} we formally define the online model.
%In the next two sections we explore the case of $k=1$ centers and constant $k>1.$
%As the paper proceeds the algorithms and lower bounds are becoming more and more complex. 
%The first algorithm is simple and includes only one line (take the first point), while the last few algorithms contain different phases and delicate proofs are needed. 
%All the proofs appear in the Appendix. 
%In Section~\ref{sec:related_work} we discuss related work and in Section~\ref{sec:conclusions_open_problems} we conclude with open problems. 
\subsection{Related work}\label{sec:related_work}
%Liberty, Sriharsha, and Sviridenko 
\cite{liberty16} presented an algorithm for online $k$-means where centers decisions are irreversible. The order is arbitrary and the cost of a point $x$ is with respect to the closest center in the set of centers selected till $x$'s arrival. 
%that uses the same online clustering setting as this paper was described. %where centers are chosen or not and a point cannot be chosen or unchosen again. 
Their algorithm adapts the $k$-means++ algorithm by \cite{arthur07} to the online case.  Inherently, their algorithm cannot get the optimal bound in the no-substitution setting, as the number of centers depends on the aspect ratio, which can be arbitrarily large. See more details in Appendix~\ref{apx:k_2_random_order_unknown_n}. %, in contrast to our algorithm that is able to achieve optimal results. 
In this paper, we improve both the approximation and the number of centers to the optimal values (see Algorithm~\ref{alg:k}), assuming the order is random.  If the order is arbitrary, then we prove that any approximation algorithm, in the worst case, needs to take almost all points as centers.

  A recent work, \cite{hess19}, designed an algorithm that bears some similarity to  Algorithm~\ref{alg:k_2_upper_known_n}. 
 %At a high-level view, their algorithm has two phases: (i) Take a few random points and find a clustering (ii) take centers that are close to centers that are chosen in the first phase. 
 %Similarities: they considered the same center-choosing model as ours. 
However, \cite{hess19} considered the statistical question where there is an underlying distribution, as in \cite{ben07}. In this statistical setting the ordering is not a factor. Also, they have to assume that the example space is bounded, and this exclusion of outliers simplifies the solution. %This means that there are not any far points. 
%or small cluster, 
%In this case the analysis is much simpler, and they do not need to add a different phase as we had. Also note that they allow for an additive error term, which we do not allow (see Definition~\ref{dfn:c-approximation}). 
 
%In \cite{ackerman14} the setting was incremental clustering, where points are observed one after another and at each step, and the number of bits used by the algorithm is small. 
%There are a few differences between our setting and the one of \cite{ackerman14}. 

%In \cite{ackerman14}, which uses the incremental clustering setting, a point can be chosen as a center after many points were observed, while in our setting this is prohibited. Additionally, their 

%Several works researched online clustering in different settings. %In \cite{ackerman14} the goal is to recover structure in the data.
 
%Our goal is to design algorithms that are competitive with the optimal algorithm. 
%where at each step a point is observed and then a clustering is chosen. They showed a few lower bounds and upper bounds in the case that the algorithm uses a bounded number of bits. They had several an assumptions on the dataset. 
In the streaming model (\cite{aggarwal07,guha03,charikar03,ailon09,shindler11,har2004coresets,phillips2016coresets}) points arrive one after another. %, as in our online setting. 
%However, there is an important difference between the two models. In the streaming model, it 
But, unlike our setting, the algorithm is allowed to choose a center after new points were observed and even go over the points a few times. %, while we do not allow it. 
%a few passes over the dataset is allowed. 
%and saving a small number of bits 
%Our setting is more restrictive because in the streaming model centers can be chosen after they have been observed, while in our setting this is not allowed.
%Many works \cite{guha03,charikar03,ailon09,shindler11,braverman11} designed algorithms in this setting (in other words, they proved upper bounds) for $k$-means clustering. We give a full picture by providing lower and upper bounds that are optimal in the setting described in Section~\ref{sec:on_line_model}. 
%We have a more restrictive setting as a point cannot be a center after a different point has received. 
 \cite{braverman11,ackerman14,raghunathan2017learning} assume that the data has some structure, %(``nice'' or ``perfect'') which they wish to discover. 
we, however, do not have any assumptions on the data and our algorithms function correctly under any dataset.

In the online facility location, \cite{meyerson01}, points arrive one at a time, and a set of facilities $F$ is maintained throughout. Each point $p$ incurs instant cost, $d(p,\ell)$, by its closest location $l\in F$. The total cost is $|F|+\sum_p d(p,\ell).$
In our setting, the cost incurs only at the end, but most importantly, we want to minimize the number of centers conditioned on having $O(1)$-approximation. In online facility location, if the distances are too small or too big, then one of the terms, $|F|$ or $\sum_p d(p,\ell)$ can dominate over the other.  
Several variants of this problem were investigated (e.g., \cite{lang18,feldkord18}).

% In the secretary problem \cite{ferguson89}, a sequence of candidates are received one-by-one. Every candidate is associated with a value. The algorithm can chose only one candidate, aiming to maximize the value of the chosen candidate.
% Recently \cite{kaplan19} suggested a modified framework for the known secretary problem. \cite{cohen19}

%\section{The Online Model}
\section{Preliminaries}\label{sec:on_line_model}
In this paper we fix the desired number of clusters to be some constant $k$. We want to design algorithms that minimize the $k$-means cost. When the algorithm is understood from the context we denote its cost by $cost(alg)$. We focus on $\Theta(1)$-approximation algorithms, which are formally defined next.  
%some constant $k$ that represents . We wish to find an algorithm that minimize the cost, we denote by $cost(alg)$ the
\begin{definition}[$a$-approximation]
\label{dfn:c-approximation}
We say that a clustering algorithm is an \emph{$(a,k)$-approximation}, $a\geq1$, for $opt_k$ if for every series of $n$ data points with probability at least $0.9$  $$\frac{cost(alg)}{cost(opt_k)}\leq a,$$ when $k$ is understood from the context we simply write \emph{an  $a$-approximation algorithm}. 
\end{definition}
In the paper, we focus on the case that $a$ is some constant, and the goal is to minimize the number of centers. The complementary problem of fixing the number of centers will lead to an infinite approximation in some cases, as our lower bounds suggest. We focus either on a fixed order of examples or random (uniform) order. Note that there are two possible sources of randomness: the algorithm and the points' order. The algorithm should succeed with probability $0.9$  (this is some arbitrary constant close to $1$) when considering the two sources together.

\section{The curious case of \texorpdfstring{$k=1$}{Lg}}\label{sec:k_1}
In this section, we focus on the case that there is only one center in the optimal clustering, i.e., $k=1$. The goal is to find one good enough center. In the offline setting, this problem is trivial, simply take $\frac{1}{n}\sum_{i=1}^n x_i$, or a point that is closest to it as the center. So it is surprising that in the online case there is a complex behavior. 

It is known that a random point in a cluster is a good enough center of the entire cluster (see Lemma~\ref{lemma:random_point_in_cluster} in the appendix). Thus, if the order is random, the algorithm can simply take the first point as a center. If the order is adversarial, but $n$ is known in advance, then a random number in $[n]$ can be taken before the examples were observed. This gives access to a random point, which we know is a good center. For completeness, the proofs of these claims are in Appendix~\ref{apx:proofs_main_theorems} as Claims~\ref{thm:k_1} and \ref{thm:k_1_worst_n_known}. 

% \begin{enumerate}
% % \item The order of the examples dramatically affects the results.
% % \item If the learner knows $n$ in advance, this also dramatically affects the results.
% % \item Some orders of the examples require $\log n$  centers to achieve constant approximation, while for other  scenarios $\Theta(1)$ is enough.
% \item If the order of the examples is random, $O(1)$ centers suffices. 
% \item If the learner knows $n$ in advance, then again $O(1)$ centers suffices. 
% \item If none of the above cases hold, then  $\Theta(\log n)$ centers is both required and sufficient.
% \end{enumerate}

%\begin{algorithm}
%\caption{On line clustering with $k=1$, random order, $n$ unknown}
%\begin{algorithmic} 
%\STATE take the first point as a center
%\FOR {$t=2,\ldots,n$} 
%        \STATE do nothing
%\ENDFOR
%\end{algorithmic}
%\end{algorithm}

In case that $n$ is unknown in advance, then $O(\log_c n)$ centers are sufficient to achieve an $O(c)$-approximation, by applying the doubling method, see more details in Claim~\ref{thm:k_1_worst_n_unknown_upper}, Appendix~\ref{apx:proofs_main_theorems}. We prove that for any $c>1$, any algorithm must take $\Omega(\log_c(n))$ centers for it to be a $c$-approximation. This means that $\Theta(\log_c(n))$ is tight for any $O(c)$-approximation algorithm.

\begin{theorem}\label{thm:k_1_worst_n_unknown}
For any integer $n$ and $c\geq1$, and for any clustering algorithm that is not given $n$ in advance and is a $c$-approximation, there are $n$ data points and an ordering of them such that the algorithm must take $\Omega(\log_c(n))$ centers with probability at least $0.8$.  
\end{theorem}

\begin{wrapfigure}{r}{0.5\textwidth}
% \begin{figure}[!htb]
\centering
\begin{minipage}{.5\textwidth}
\begin{center}
%\begin{figure}
        \includegraphics[scale=0.5]{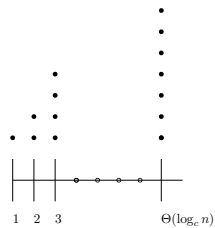}
        \caption{Dataset for proof of Theorem~\ref{thm:k_1_worst_n_unknown}}
        \label{fig:k_1_lower_bound_worst}
%\end{figure}    
\end{center}
\end{minipage}
\end{wrapfigure}
%We remark that the ``magic number" $0.9$ is the same constant the appears in the definition of a $c$-approximation, Section~\ref{sec:on_line_model}, 
We remark that the constant $0.8$ is merely a number smaller than $0.9,$ which appeared in the definition of a $c$-approximation, Definition~\ref{dfn:c-approximation}. 
One cannot prove that an algorithm must take $\Omega(\log_c(n))$ centers with a probability larger than $0.9$ because a valid approximation algorithm can decide with probability $0.1$ not to take any center. 
The idea of the proof 
%of Theorem~\ref{thm:k_1_worst_n_unknown} 
is to construct a dataset and an order on them such that the number of centers taken is $\Omega(\log_c n)$ for any $c$-approximation algorithm. 
The dataset is composed of $\Omega(\log_c n)$ groups. The groups are evenly spaced on the line (see Figure~\ref{fig:k_1_lower_bound_worst}). The number of points in each group is exponential increasing. The points are given, group by group from smallest to largest.
%To achieve $c$ approximation, the learner must take a center from the last group. 
 Since the algorithm does not know $n$, the current group can be the last and recall that an approximation algorithm must succeed under any dataset. Thus, the algorithm must take a center from each group.  The formal proof is in the appendix. 
%The lower bound we just proved is tight, as the next theorem proves. 

\section{The case of constant \texorpdfstring{$k\geq2$}{Lg}}\label{sec:bounds_k_2}
% The results in this sections are:
% \begin{enumerate}
% \item (arbitrary order) There is an order of the dataset such that any clustering algorithm must choose $\Omega(n)$ points as centers, which is tight. 
% \item (random order) For a random order and known $n$, $\Theta(1)$ centers are sufficient, which is tight. 
% \item (random order) For a random order and unknown $n$, $\Theta(k\log \frac{n}{k})$ centers are needed and sufficient. Again, this is a tight result.
% \end{enumerate}
%\subsection{Arbitrary order}
%The next theorem shows that 
This section explores the case where the optimal clustering contains $k$ centers, where $k>1$ is any constant. If the dataset's order can be arbitrary, then any $c$-approximation deterministic algorithm must take all the points in the dataset. A dataset that shows this is $n$ non-negative points on the line, i.e., $x_t\in\mathbb{R}$, in increasing order where each point is much further than the previous one. Since it is so further away, it has to be taken, otherwise, the rest of the points will be set to $0$, which is allowed as points arrive in an arbitrary order. In this case, the largest point has to be taken for the algorithm to be a $c$-approximation. Thus, all points need to be taken. If the algorithm is stochastic, it needs to take $\Omega(n)$ points as centers. See \cite{liberty16} or Claim~\ref{thm_k_2_worst} in Appendix~\ref{apx:proofs_main_theorems} for more details. 
As a side note, a follow-up work \cite{bhattacharjee2020}, proved that for ``structured" data (e.g., points sampled from a $k$-mixture model), $\poly(k\log n )$ centers are enough to achieve $O(k^3)$ approximation. 
As for the upper bound, an algorithm can take all $n$ data points as centers and achieve a minimal cost of $0$. Thus, the upper and lower bounds coincide, up to a constant, when the order is arbitrary.

\subsection{Random order and known \texorpdfstring{$n$}{Lg}}\label{subsec:random_order_known_n}
Now let us assume that the data arrives in random order. If $n$ is known in advance, we show an algorithm that takes $\Theta(1)$ centers and is a $\Theta(1)$-approximation, for any constant $k$. The main idea is to observe a small linear fraction of points without taking any as a center. The option of merely observing data without taking points as centers is possible only because $n$ is known in advance, and the points' order is random. Thus, with a high probability, good enough centers will also be available in the future. Fortunately, as the order is random, the small sample provides enough information. %on the large optimal clustering. 

The algorithm is composed of three phases. In the first phase, it observes a small linear number of points, $M_1$, without taking any point as a center. It finds an approximately optimal clustering for the points $M_1$.
%It finds $k$ centers that are approximately optimal clustering for the points $M_1$. 
These centers define a clustering $C^{M_1}$ on the entire dataset, where each point is clustered with its closest center. Since the order is random, these centers are good centers for all large clusters, of size at least $\poly(k)$, in the dataset, see Claim~\ref{clm:known_bound_cost_large_cluster_exists_good_cluster} in the appendix. Unfortunately, these points cannot be taken as centers in retrospect in our framework. So the algorithm needs to take future points that are close to those centers. %It turns out that if the clusters are far apart, many points are close enough to the centers in $C^{M_1}$, see  Claim~\ref{clm:k_2_known_n_large} in the appendix. So taking the first few points from each cluster is good enough, as the points' order is random. If the clusters are not far apart, some optimal clusters can be merged, Claim~\ref{clm:k_2_known_n_large}. 

%, so the algorithm takes, among the remaining points, centers that are close to the centers found in the first phase. 
%  To help define the ``close enough" points, we use the fact, see Section~\ref{sec:k_1}, that for any cluster $C^*_i$, there are many points inside the cluster $Good_i\subseteq C^*_i$ that can be a good enough center for the entire cluster $C^*_i$. 
 
%  The algorithm needs to pick one of them. 
 
%  - There are many points that are close enough 
%  - Taking the first is enough 
%  - Taking a few is easier 

%  - Significant fraction of points are in the first phase
%  - significant number of points are not chosen in phase one, and can be chosen as centers 
%  - The first part gives claim 30 
%  -  

% - Any point in $Good_i$ is close enough, there are many points that are close enough 
% - Taking just the first point is good enough 
% - Taking more points will increase the probability 

% - Emphasize that the clustering found is a good representative, because the order is random 
% - The algorithm takes a few close points. 
% - Easier to analyze, Why take a few and not just one? 
% - because it's random, the first is good is close to center 

The clustering $C^{M_1}$ represents well only large clusters. To see why, consider, as an extreme example, a small cluster that contains only one point. The algorithm has to take this point as a center when received, or a high cost is incurred. In other words, the algorithm has to take points that are ``far". We define ``far" by farthest from their center $c_i$ by some threshold $t_i$. The algorithm needs to decide how to define this threshold, where one option is to take $t_i$ as the max radius of cluster $C^{M_1}_i$. However, this is problematic, as it might cause the algorithm to take too many centers. The problem stems from the fact that we cannot use the sample to define both the centers and $t_i$'s, as the centers are selected to minimize $t_i$'s. To overcome this problem, we introduce an intermediate step, phase $2$. 

In phase $2$, we save another small fraction of points, without taking any of them as center. The distance to the farthest point, $R_{max}[i]$, in this sample, from each center, defines the threshold $t_i=R_{max}[i]$. Since $t_i$'s and the centers are now independent, this will guarantee that the algorithm does not take too many centers that seem far, see Claim~\ref{clm:k_2_upper_known_n_bound_centers}.    
In phase $3$, we finally take centers. There are two types of centers: (i) points that are close to cluster centers from phase $1$, or (ii) points that are considered ``far''.   

To summarize, there are three phases in the algorithm:  \begin{itemize}
    \item \textbf{Phase 1:} The first $\alpha n$ of the points, $\alpha\in (0,1)$ is a constant to be chosen later, are saved in memory, and the algorithm does not take any of them as centers.
We denote this set by $M_1.$
Since $n$ is known, the algorithm can decide not to take $\alpha n$ of the points as centers without increasing the cost by much. 
After the first $\alpha n$ points arrive the algorithm uses them to find $k$ centers $c_1^{M_1},\ldots, c_k^{M_1}$ that are $\Theta(1)$-approximation clustering for the $\alpha n$ observed points. 
\item \textbf{Phase 2:} 
 The algorithm observes another $\alpha_2n$ points, $M_2$, without taking any as center, $\alpha_2\in (0,1)$ is another constant to be chosen later. For each center $c_i^{M_1}$ it saves the distance to the farthest point, $R_{max}[i]$, in its cluster among those in $M_2$. %A far point is one that is further than these $\alpha_2 n$ points. 
\item \textbf{Phase 3:} 
The algorithm takes the following centers for each center $c_i^{M_1}$:
(i) a few close points 
(ii) points that are farther than the threshold $R_{max}[i]$. %found in phase $2$.
%\begin{enumerate}
%\item 
%(i) A few close points to each of the $k$ centers 
%\item 
%(ii) Points that are very far away, these points can form a cluster of their own.
\end{itemize}
 
%\end{enumerate}
The algorithm's pseudo-code is in Algorithm~\ref{alg:k_2_upper_known_n}, and its correctness is proved in the next theorem. 
\begin{algorithm}[t]
\caption{Online clustering with $k>1$, $n$ known, random order}
\begin{algorithmic}[1]\label{alg:k_2_upper_known_n}
%\STATE \underline{data points appear in worst-case order, $n$ is known}
 \STATE {\color{titleColor}{\textbf{phase 1: collect data} }}
 \STATE $M_1=$ save (without taking as center) the first $\floor{\frac{n}{10^2k} }$ points
% \STATE $M=\emptyset$
% \FOR {$t=1,\ldots,\floor{\frac{0.01}k n}$} 
%\STATE add $x_t$ to $M$ 
%\ENDFOR
%\STATE max\_dis $:= \max_{y\in M} \norm{x-y}$
\STATE find offline clustering for $M_1$ with %clusters $(C^{M_1}_i)_{i=1}^k$ and
centers $(c^{M_1}_i)_{i=1}^k$
%\STATE $R_{max}[i] = \max_{y\in C^M_i}\norm{y-c^M_i}$   (for each cluster save max distance)  
%\STATE $R[i] = \E_{y\in C^M_i}[\norm{y-c^M_i}^2]$  (for each cluster save average distance)   \label{alg:k_2_upper_known_n_line_average}
 \STATE {\color{titleColor}{\textbf{phase 2: collect more data to define ``far" points}}} 
 \STATE $M_2 =$ save (without taking as center) the next $\floor{\frac{n}{10^5k^3}}$ points
 \FOR {$i=1$ to $k$}
 \STATE $A[i] = \{x \in M_2: i=\argmin\|x-c_i^{M_1}\|\}$ \COMMENT{ partition  $M_2$}  %set of points in $M_2$ that their closest center is $c^{M_1}_i$
 
  \STATE $R_{max}[i] = \max_{y\in A[i]}\norm{y-c^{M_1}_i}$ \COMMENT{ max distance between $M_2$ and} 
   \STATE \COMMENT{ current center, if $A[i]=\emptyset$, this is $0$}  
   %\IF {$A[i]\neq\emptyset$}
%  \STATE $R_{max}[i] = \max_{y\in A[i]}\norm{y-c^{M_1}_i}$   \COMMENT{maximal distance between $M_2$ and current center}  
%  \ELSE 
%  \STATE $R_{max}[i] = 0$
%   \ENDIF
 \STATE $centers\_counter[i] = 0$ \COMMENT{ init close-centers-counter for phase 3}
 \ENDFOR
 \STATE {\color{titleColor}{\textbf{phase 3: take centers}}}
 \FOR {the rest of the points $x_t$ (points not received in phase 1 or 2)} 
 \STATE $i^* =\argmin_i \norm{x_t - c^{M_1}_i}$ \COMMENT{closest center}
	%\IF {$\norm{x_t-x}>$ max\_dis} 
	\IF {$\norm{x_t-c^{M_1}_{i^*}}>R_{max}[i^*]$}
	  %\STATE
        \STATE take $x_t$ as a center  \COMMENT{points that are far away}
       \label{alg:k_2_upper_known_n_line_take_far_point}
       % \STATE max\_dis $:= \norm{x_t-x}$
       \ENDIF
       \IF {$centers\_counter[i^*] \leq 3k\log(40k)$} 
       %and $\norm{c^M_i-x_t}\leq R[i]$}
       \STATE take $x_t$ as center 
 \label{alg:k_2_upper_known_n_line_take_close_point}
        \COMMENT{not enough close points to center $c^{M_1}_i$}
       \STATE \COMMENT{ were taken yet $\Rightarrow$ taking a close point}
       \STATE $centers\_counter[i^*]++$
       \ENDIF
\ENDFOR
\end{algorithmic}
\end{algorithm}

%\begin{theorem}\label{thm:k_2_upper_random_known_n}
%For any integer $k \geq 2$, there is an algorithm that given the size of the datata, $k$, and the dataset in a random order, with probability at least $0.9$ uses $O(k^3)$ centers and is $O(k^4)$-approximation comparing to $opt_k$. 
%\end{theorem}

\begin{theorem}\label{thm:k_2_upper_random_known_n}
For any constant integer $k\geq 2$, there is an algorithm that given (i) $n$, the size of the dataset, (ii) $k$, and (iii) the dataset which appears in a random order, the following holds. With probability at least $0.9$, the algorithm takes $\Theta(1)$ centers and $cost(alg)\leq \Theta(1)\cdot cost(opt_k)$.
\end{theorem}

To prove the theorem, we need to bound the number of centers the algorithm takes and its approximation. To bound the number of centers,  we note that the algorithm takes two types of centers in Line~\ref{alg:k_2_upper_known_n_line_take_far_point} and Line~\ref{alg:k_2_upper_known_n_line_take_close_point}. It is easy to bound the second type of centers by $O(k^2\log k)$. To bound the first type, focus on one cluster $C^{M_1}_i$ in the clustering $C^{M_1}.$
 We prove that a significant fraction, $a$, of the points in $C^{M_1}_i$ are received in phase $2$, see Claim~\ref{clm:random_from_each_part}. We prove that the probability of taking a point as ``far" is inversely proportional to
$a|C^{M_1}_i|$. The multiplication of the last two terms bounds the expected number of points taken as ``far" points. Importantly, this multiplication is a constant. See Claim ~\ref{clm:k_2_upper_known_n_bound_centers} for more details.

%-- separate between large and small clusters : maybe only for the analysis? 

Next, we want to prove that the algorithm is a $\Theta(1)$-approximation, which is formally proved in Claim~\ref{clm:k_2_upper_known_n_bound_cost} in the appendix.  We show that for each optimal clustering, the algorithm takes as center a point that is a good enough center for the entire cluster.  For that aim, we separate the analysis into two cases depending on the size of the cluster: small (of size smaller than $\poly(k)$) or large. We start with the small-size analysis. Focus on a small cluster $C^*_i$. Take a point in $x_0\in C^*_i$ that is a good center for all the points in $C^*_i$. As a side note, since the cluster is small, there might be only one point in $C^*_i$. This point, most likely, will not be received in the first two phases. Suppose that $x_0$ is in a cluster $C^{M_1}_i$ with center $c^{M_1}_i$. Let us focus on all points $A$ in $C^{M_1}_i$ that are farther than $x_0$ from the center.  It $A$ is small, then none of the points in $A$ are chosen in phase 2, see Claim~\ref{clm:small_random_set}, and $x_0$ will be taken as a center. If $A$ is large, then $C^*_i$ can be merged into a different optimal cluster. This is formally proved in Claim~\ref{clm:k_2_known_n_small}.

Moving on to the case of large optimal cluster $C^*_i$, from Section~\ref{sec:k_1}, we know that most points $Good_i\subseteq C^*_i$ in the cluster can be a good enough center for the entire cluster. This implies that the fraction of points we get from $Good_i$ in phases $1$ and $2$ is between $(\alpha+\alpha_2)/2$ and $2(\alpha+\alpha_2)$ as the order is random, see Claim~\ref{clm:random_from_each_part}. Focus on the cluster $C^{M_1}_i$ with the center $c^{M_1}_i$ containing most of the remaining points from $Good_i$. There are two cases: either $C^{M_1}_i$ includes mostly points from $Good_i$ and then the algorithm probably takes a point from $Good_i$, or this cluster $C^*_i$ can be merged into a different optimal cluster. This is formally proved in Claim~\ref{clm:k_2_known_n_large}.

%The proof idea is that the points collected in the first phase gives enough information to find good enough centers for all the large clusters (of size at least $\poly(k)$), as the order is random. 
%Since the order is random, the points collected in the first phase are representative of the entire dataset.
%Since $n$ is known, the algorithm can take a small portion of the points and be sure that more points that can be good centers will be received in the third phase. So there is no harm, for the large clusters, in not taking as center the first $\alpha n$ points. 
%Importantly, with probability more than $0.9$, good centers for the small clusters will be received in phase $3$, since the number of points received in phase $1$ and $2$ is small, only  $(\alpha+\alpha_2)n$. 
%These small clusters must be very far away, otherwise, they will not be small. Thus, points from small clusters will be taken in the third phase. 

Before moving to the following case, a few remarks. The paper does not try to optimize the dependence on $k$, where the number of centers is $\poly(k)$, and the approximation is $\exp(k\log k)$. Indeed, in a follow-up work \cite{hess21}, a new algorithm was presented with improved dependency on $k$. Second, the work \cite{indyk99} designed a sublinear time algorithm for $k$-medians, which has some similarities to Algorithm~\ref{alg:k_2_upper_known_n}. One major difference is the algorithm's treatment of far points. While they can consider the furthest points in the entire dataset as far points, we need to decide online if a point is far or not. For that, we
had to introduce phase $2$.
%define a threshold for each cluster that defines what a far point is.  
A detailed discussion of more differences can be found in Appendix~\ref{apx:k_2_random_order_known_n}. 

%, which is a similar to the $k$-means problem discussed in this paper.
% Their algorithm 

 %to one of the algorithms described in this paper,

 %, that works in the case that number of points in the dataset, $n$, is known in advance and the order of the dataset is random.  
%  At a very high level, both algorithms have two phases (i) solve the problem for large enough clusters using a few random points (ii) take as center points that are far. 
%  But, there are a few significant differences between their algorithm and ours, chief among them  %Defining the threshold as the maximal distance between the points received so far and the centers will result in a biased estimate. To resolve this issue we introduce an intermediate phase that includes saving new points that define an unbiased threshold. 

 %that cause the analysis (and algorithm) to be different: 

 \subsection{Random order and unknown \texorpdfstring{$n$}{Lg}}\label{subsec:random_order_unknown_n}
 In the last section, we designed an algorithm that uses $\Theta(1)$ centers and achieves $\Theta(1)$-approximation, when $k$ is a constant, if the number of points, $n$, is known in advance. In contrast, in this section, we show that if $n$ is unknown in advance, any algorithm must take $\Omega\left(k\log \frac{n}{k}\right)$ centers. 
 The lower-bound dataset is similar to the $\Omega(n)$ lower bound used in the worst-case order, where points are in $\mathbb{R}$ with increasing distances. The idea is that an approximation algorithm must take the $k-1$ largest points as each step; otherwise, the data stream can stop. 
  In the rest of the section, we design a new algorithm that achieves a matching upper bound, up to a constant.
 
%  The lower-bound dataset lays in $\reals^{k}.$ We take the worst-order lower bound, where we had a series of points in $\reals_{\geq 0}$ with exponential differences and create $k$ such series each with $n/k$ points. Each series $i\in[k]$ will correspond to one dimension $i$ where all coordinates are zero everywhere except dimension $i$. Perhaps surprisingly, this dataset was used in \cite{lattanzi17} to show a lower bound for a different problem of consistency between consecutive clustering results.
%   In the rest of this section, we design a new algorithm that achieves a matching upper bound, up to a constant.
 
 \begin{theorem}\label{thm:k_lower_random}
For any scalar $c>1$, integers $k\geq 2$ and $n$, and for any clustering algorithm that does not know what $n$ is and is a $c$-approximation, there are $n$ points that arrive uniformly at random and the algorithm must take $\Omega\left(k\log \frac{n}{k}\right)$ centers with probability at least $0.7$.
%For any $c>1$, any learner that is a $c$-approximation must use at least $\Omega(k\log \frac{n}{k})$ centers when $n$ is unknown.  
\end{theorem}

\subsubsection{The case of \texorpdfstring{$k=2$}{Lg}}
%In this section we explore the case that $k=2.$ 
To simplify the presentation, we start with the case that $k=2$. Many of the ideas are also applicable to the case of $k>2.$
We prove that $\Theta(\log n)$ centers are needed and sufficient for a $\Theta(1)$-approximation, when $n$ is unknown and $k=2$.
 We show a simple algorithm that saves only a small number of bits in memory (more specifically, it saves the first example and only one more number), achieves $\Theta(1)$-approximation, and takes at most $\log(n)+2$ centers. Our results are tight when $n$ is unknown, as we show a matching lower bound. 

%\subsubsection{Lower bound} 
%\textbf{Lower bound}

%\subsubsection{Upper bound} 
%\textbf{Upper bound} 

% %We now move on to prove a matching upper bound. 
% We show a simple algorithm  and by choosing only $O(\log n)$ centers it is able to achieve  $\Theta(1)$-approximation. 

\begin{algorithm}
\caption{Online clustering with $k=2$, $n$ unknown, random order}
\begin{algorithmic}[1]\label{alg:k_2}
%\STATE \underline{data points appear in worst-case order, $n$ is known}
 \STATE take $x_1$ as a center
\STATE $x := x_1$ (save first data point)
\STATE max\_dis $:= 0$
\FOR {$t=2,\ldots$} 
	\IF {$\norm{x_t-x}>$ max\_dis} \label{line:is_new_cluster}
        \STATE take $x_t$ as a center
        \STATE max\_dis $:= \norm{x_t-x}$
       \ENDIF
\ENDFOR
\end{algorithmic}
\end{algorithm}

\begin{theorem}\label{thm:k_2_upper_random}
There is an online algorithm such that if the examples are received with random order ($n$ does not have to be known) then with probability at least $0.9$ it holds that number of centers is $O(\log n)$ and $cost(alg)\leq \Theta(1)\cdot cost(opt_2).$
\end{theorem}

To bound the number of centers taken by Algorithm~\ref{alg:k_2}, note that the $i$-th example is chosen as a center only if it is the furthest from $x$ (recall that $x$ is the first example). This will happen with probability $\frac{1}{i-1}$. Thus, the expected number of centers is the $n$-th harmonic number, which is about $\log n.$ 
To prove the algorithm is a $\Theta(1)$-approximation, in a high level, we separate the analysis into two cases: either the two clusters are close to each other or not. It the two clusters are close, we can treat them as one cluster with center $x$. Using Lemma~\ref{lemma:random_point_in_cluster}, $x$ is a good center. If the two clusters are far apart, then, most probably, the first point from the second cluster is furthest away from $x$ among all points received so far. 

Formalizing the last argument, we want to show that the algorithm takes two points as centers that are good representatives of each of the optimal clusters $C^*_1,C^*_2$. Denote by $Good_i$, $i=1,2$, the set of points in each optimal cluster $C^*_i$ that can be taken as a center without increasing the cluster's cost by much.  From the same arguments as in Section 3, we know that $Good_i$ is a significant fraction of $C^*_i$. 
%many points in each optimal cluster $Good_1\subseteq C^*_1,Good_2\subseteq C^*_2$ are good centers to their cluster.
With high enough probability, the first point from each cluster $i=1,2$ is a good center, i.e., in $Good_i$. Specifically, $x$, the first point, is a good center for its cluster $C^*_1$. Thus the algorithm needs to take as a center one point in $Good_2$ (the algorithm might take many more points as centers to achieve this goal). We hope to show that the algorithm takes the first point from $C^*_2$ as a center. Denote by $y^*_2\in Good_2$ the closest point in $Good_2$ to $x$. Focus on the set $B$ of points that will interfere in taking the first point in $Good_2$ as a center $$B = \{y_1\in C^*_1 : \norm{y_1-x} \geq \norm{y_2^*-x} \}.$$
There are two cases: either $B$ is small compared to $C^*_2$ or not. If it is small, then most likely, the first point from $C^*_2\cup B$ is in $C^*_2$, or in different words, the first point from $C^*_2$ will arrive before any point in $B$. Thus the first point from $C^*_2$ will be taken as center. In the other case, $B$ is large compared to $C^*_2$. This means merging $C^*_1$ and $C^*_2$ together increases the cost by only a constant factor. Thus, $x$ can be a good center for $C^*_2$ too.

%The idea: the algorithm found a representative of the first cluster by taking the first point $x.$
%The second cluster should be further a way from the first cluster.
% So points that are further away from $x$ are candidates to be in the second cluster. 
% 
%To elaborate on that, there are two cases either the two clusters are far away from each other, and then 
%$$\Pr_{\substack{y_1\in C^*_1 \\ y_2\in C^*_2}}\left(\norm{y_2-x} > \norm{y_1-x}\right)$$ is large and points in $C^*_2$ are further away from points in $C^*_1,$ and this is why the algorithm chooses a point from the second cluster. Another possibility is that the two clusters are close to each other. In this case, we can join the two clusters into one cluster with $x$ as the center and the cost of the new cluster will be at most $\Theta(1)\cdot cost(opt).$

%The idea is that the we want see a good center from one of the clusters only if there are many outlier in the first cluster, but then we can just use one center for both clusters (explain better!)

\subsubsection{The case of constant \texorpdfstring{$k>2$}{Lg}}
%The lower bound with worst case order still holds. 
%\textbf{Lower bound}

%\subsection{Upper Bound}
%\textbf{Upper bound}

%We now move on to the upper bound when $k$ is a constant. 
For the more general case of constant $k>2$ we present Algorithm~\ref{alg:k} that uses $O\left(k\log\frac{n}k\right)$ centers, which matches the lower bound of Theorem~\ref{thm:k_lower_random}, and is a $\Theta(1)$-approximation for constant $k$. 

We want to borrow the main idea of Algorithm~\ref{alg:k_2}: if clusters are far apart, take the first point from each optimal cluster, otherwise merge clusters. Algorithm~\ref{alg:k_2} detects the arrival of the first point $x_t$ from a new cluster by measuring the distance to the first point received $x$, see Line~\ref{line:is_new_cluster}. This technique will not work for $k>2$, as the next example demonstrates. Focus on $k=3$ and three well-separated clusters on a line with centers $a_1 \ll a_2 \ll a_3$, where the middle cluster is much smaller in size than the other two clusters, see Figure~\ref{fig:k_motivation_algorithm}. Most likely, the first points will be from the first and third clusters. It is unclear how to detect the arrival of the first point from the second cluster. 

Inspired by Figure~\ref{fig:k_motivation_algorithm}, taking $k$ points that are farthest from each other, the first point from a new cluster is one of those $k$ points. The main idea of Algorithm~\ref{alg:k_2} is to use the known farthest-first-traversal algorithm as a subroutine. Perhaps surprisingly, This subroutine is beneficial for a different cost function, $k$-center \cite{dasgupta2013geometric}. For completeness, the farthest-first-traversal algorithm appears as Algorithm~\ref{alg:fft}.

\begin{center}
\begin{figure}
       \centering \includegraphics[scale=0.5]{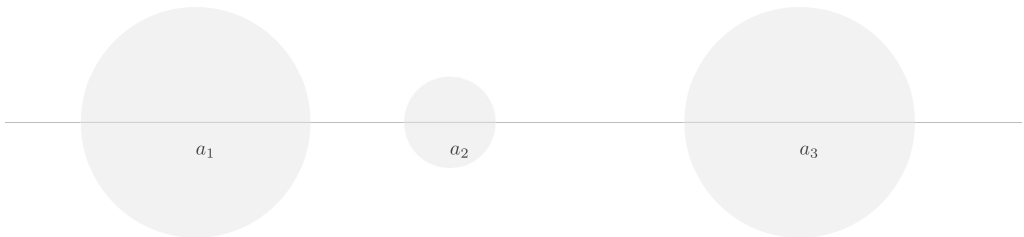}
        \caption{Motivation for Algorithm~\ref{alg:k}: detecting the middle cluster.}
        \label{fig:k_motivation_algorithm}
\end{figure}    
\end{center}
%The algorithm uses the known farthest-first-traversal algorithm as a subroutine, for completeness it appears as Algorithm~\ref{alg:fft}.
\begin{algorithm}
\caption{Farthest-first-traversal$(M,s,k)$}
\begin{algorithmic}[1]\label{alg:fft}
\STATE $S=\{s\}$
\FOR {$t=2,\ldots,k$} 
	\STATE $v = \argmax_{x\in M - S} \min_{y\in S} \norm{x-y}^2$
	\STATE $S = S \cup \{v\}$ 
\ENDFOR
\RETURN $S$
\end{algorithmic}
\end{algorithm}

The farthest-first-traversal algorithm returns $k$ (a parameter) points that are far away from each other in a given dataset $M$. 
Specifically, it starts with some point $s\in M$ that is given as an input. 
Then it takes the point $x_2\in M$ that it furthest away from $s.$
Then a point $x_3\in M$ that maximizes the distance to $S$, where the distance is equal to  $dis(S,x)=\min_{y\in S}\norm{y-x}.$

For our purposes, the primary claim we need from farthest-first-traversal is that if $S$ was returned, then the distance of any point $x$ to $S$ is smaller than the distance between any two points inside $S$, as the next lemma proves. For completeness, the proof is in Appendix~\ref{apx:auxiliary_claims}.
\begin{lemma}\label{lemma:farthest_first_traversal}
Suppose  $S=$ Farthest-first-traversal$(M,s,k)$ and $x\in M-S$, then 
$$\min_{y\in S}\norm{x-y} \leq \min_{y_1,y_2\in S} \norm{y_1-y_2}.$$ 
\end{lemma}

The online clustering algorithm, Algorithm~\ref{alg:k}, saves in memory all the points encountered so far. Deciding whether to take the current point as a center or not uses the farthest-first-traversal algorithm that picks $k$ far away points. The current point is taken as a center if chosen as one of the $k$ points.

\begin{algorithm}
\caption{Online clustering with $k\geq 2$, $n$ unknown, random order}
\begin{algorithmic}[1]\label{alg:k}
\STATE Take $x_1,\ldots, x_k$ as centers and save them in $M$
%\STATE $M = (x_1,\ldots, x_k)$
\FOR {$t=k+1,\ldots$} 
\STATE $M = append(M, x_t) $
	\STATE $S = $ Farthest-first-traversal$(M,x_1,k)$
	\IF {$x_t\in S$} 
        \STATE take $x_t$ as a center
       \ENDIF
     
\ENDFOR
\end{algorithmic}
\end{algorithm}

\begin{theorem}\label{thm:k_upper_random}
There is an online algorithm such that for any constant $k$, if the examples are received with random order ($n$ does not have to be known) then with probability at least $0.9$  the algorithm uses $O(\log n)$ centers and achieves $\Theta(1)$-approximation. %it holds that $$cost(alg)\leq \Theta(1)\cdot cost(opt_1).$$ 
\end{theorem}

%To prove the theorem we show two things. Both are similar to the proof of Theorem~\ref{thm:k_2_upper_random}. 
%First, that not too many 
The number of centers Algorithm~\ref{alg:k} takes is bounded because the $i$-th example, $i>k$, is chosen as a center only if it is one of the $k$ points defining the unique farthest-first-traversal for the first $i$ points. This will happen with probability $\frac{k}{i}$. Thus, the expected number of centers can be calculated using the $n$-th harmonic number, which is about $k\log \frac{n}{k}.$ 

%The second thing that the theorem shows is that with  probability at least $0.9$ 
 %. We want to show that the first point from this cluster $x\in C^*_i$ will be chosen as a center. 
Next, we show that the algorithm is a $\Theta(1)$-approximation. We remark that the paper does not try to optimize the dependence on $k$, where the approximation is $\exp(k\log k)$. Focus on an optimal cluster $C^*_i$ and its first point $x_i$. It can be the case that there is another optimal cluster $C^*_r$ and a point $x_r\in C^*_r$ that causes $x_i$ not to be taken as a center. This situation can happen, but we want to show that it occurs with a small probability. To do so, we show that the number of points that might interfere with taking $x_i$
as a center is small. We need to define the set of points that interfere. Intuitively, these are points that have a considerable distance compared to points in $C^*_i$. 
One option to define this considerable distance is the distance of $C^*_i$'s closest point to other clusters.
 Another is the furthest. It turns out that both of these options will not work. The furthest will not work because it is not necessarily a good center to $C^*_i$. The closest will not work because there might be many points in other clusters that are farthest from this point. We need to define this distance somewhere between the farthest and the closest. 
Lemma~\ref{lemma:farthest_first_traversal} implies that if the first point from a new cluster is not taken as a center, then there are many points in a different cluster $C^*_r$ with cost higher than merging $C^*_i$ to a different optimal cluster. See the proof of Claim~\ref{clm:k_2_upper_unknown_n_bound_cost} for more details.

%Using Lemma~\ref{lemma:farthest_first_traversal}, we know that $x$ will be chosen as  a center if it is far from points from other clusters. So, if $C'$ is far from all other clusters, $x$ will be chosen as a center, otherwise we can add $C'$ to another cluster. 
%To prove that, we show that for each cluster, either the first point from this cluster will be chosen as a center, or the entire cluster can be added to another close cluster. 

%To prove that we separate into two cases: either the two clusters are close to each other or not. It the two clusters are close we can treat them as one cluster. Using Lemma~\ref{lemma:random_point_in_cluster}, the first point is a good center. If the two clusters are far apart then the most probably the first point from the second cluster is the furthest away so far from $x.$

%A previous work \cite{liberty16} described an algorithm for online k-means that use the same online framework as this paper. %where centers are chosen or not and a point cannot be chosen or unchosen again. 

\section{General cost function}\label{sec:general_cost_function}
So far we focused on the $k$-means cost, but one can consider a more general cost:
\begin{equation}\label{eq:general_cost}
    cost(c_1,\ldots, c_k) = \sum_{x\in D} \dis{x}{c(x)},
\end{equation}
 where $d$ is a distance function\footnote{$d$ should be   (i) non negative ($d(x,y)\geq 0$) (ii) $d(x,x)=0$  and (iii) symmetric $d(x,y) = d(y,x)$ (iv) satisfy triangle inequality} and $c(x)=\argmin_{c_i\in\{c_1,\ldots, c_k\}}\dis{x}{c_i}.$ Specifically, we focus on distance function that satisfy a version of a triangle inequality: there is a \emph{constant} $D\geq 1$ such that  
\begin{equation}\label{eq:generalized_triangle_inequality}
\forall u,v,w.\quad d(u,v)\leq D\cdot(d(u,w)+d(w,v)).
\end{equation}

For the $k$-means cost $\dis{x}{y}=\norm{x-y}^2$ and indeed Claim~\ref{clm:norm_of_sum_vectors} proves that Inequality \ref{eq:generalized_triangle_inequality} holds with $D=2.$ In the $k$-medians case $d(x,y)=\norm{x-y}_1$ and Inequality \ref{eq:generalized_triangle_inequality} holds with $D=1.$ 

We show that our results, which are summarized in Figure~\ref{fig:summary_results}, hold for any cost function that satisfies Inequality~\ref{eq:generalized_triangle_inequality}. In the rest of this section we outline some of the proof ideas. First observe that the triangle inequality immediately implies a similar claim as Lemma~\ref{lemma:random_point_in_cluster}; see Appendix~\ref{apx:auxiliary_claims} for the proof:
\begin{lemma}\label{lemma:general_cost_expected_cost}
For any $x_1,\ldots,x_n,\mu$, and integer $j$ chosen uniformly at random from $[n]$, it holds that 
$$\E_{j\in[n]}\left[\sum_{i=1}^n\dis{x_i}{x_j}\right]\leq 2D\cdot \sum_{i=1}^n\dis{x_i}{\mu}.$$ 
\end{lemma}
This implies that a random point in a cluster is a good enough center for the entire cluster. 
\paragraph{The curious case of $k=1$.} In case that (i) the order is random or (ii) the order is worst case and $n$ is known, Lemma~\ref{lemma:general_cost_expected_cost} establishes that a constant number of centers is sufficient to achieve a constant approximation. In case that the order is worst case and $n$ is unknown (i) the doubling technique (see Algorithm~\ref{alg:k_1_n_unknown}) takes  $O(\log n)$ points as centers and achieves $O(cD)=O(c)$ approximation, see Claim~\ref{thm:general_k_1_worst_n_unknown_upper} (ii) assuming there are $n$ points $x_1,\ldots,x_n$ with $\dis{x_i}{x_j}=|j-i|$ (e.g., $x_i=i$), we can show a lower bound of $\Omega(\log n)$ centers, see Theorem~\ref{thm:general_k_1_worst_n_unknown_lower}. 

\paragraph{The case of constant $k\geq 2$.}
The algorithms presented in this paper, Algorithms \ref{alg:k_2_upper_known_n} and \ref{alg:k} function correctly under a general cost function when the norm is replaced with $\dis{\cdot}{\cdot}.$ 
Indeed, the proofs of Theorems~\ref{thm:k_2_upper_random_known_n} and \ref{thm:k_upper_random}  use the general cost (as defined in Equation~\ref{eq:general_cost}).
To avoid the specificity of the $k$-means cost, these proofs use Lemma~\ref{lemma:general_cost_expected_cost} heavily. 
To prove the lower bounds presented in this paper we need, for any $c>0$, a series of points $x_1,\ldots,x_n$ such that $\dis{x_{i+1}}{x_i}\geq c\cdot \dis{x_{i}}{x_{i-1}}$ (it is easy to find such a series in $\mathbb{R}$). Given this series, the proof for the general case is the same as the proofs of Claims~\ref{thm_k_2_worst} 
and \ref{thm:k_2_lower_random}. 

%\section{Conclusions and open problems}\label{sec:conclusions_open_problems}

\section{Conclusion}\label{sec:conclusions_open_problems}
% rephrase the first paragraph, less repetitions  
In this paper, we showed optimal bounds for online clustering when the number of centers, $k$, is a constant, i.e., we showed matching upper and lower bounds. 
We uncovered a complex behavior in the online setting compared to the offline setting. Specifically, in the former, new factors
%parameters 
arise: the order of the dataset and knowing in advance the size of the dataset. These factors 
%parameter 
have dramatic effects on online algorithms as illustrated in Figure~\ref{fig:summary_results}.
These bounds hold for any cost function that obeys triangle-type inequality. 
In the paper, we designed new algorithms that can learn under different circumstances. Specifically, if the order is random we designed algorithms that take as centers only $\Theta(1)$ points if $n$ is known, and $O(\log(n))$ centers if $n$ is unknown, both are optimal bounds. These algorithms work without any assumptions on the data. %In the worst order, there are cases when $\Omega(n)$ centers are needed. Thus, if the order is random, one can perform online clustering without any assumptions. 

\acks{
I thank Sanjoy Dasgupta for introducing me to the fascinating world of online $k$-means, reviewing initial drafts of this paper, and for many stimulating discussions.
}
\bibliography{onlineclustering}

\newpage
\appendix
% %\section{Proofs}
% %This is the supplementary material for the paper ``Unexpected Effects of Online K-means Clustering". 
% %This section includes the proofs for the theorems and claims that appear in the paper. 
%  %``Unexpected Effects of Online K-means Clustering". 
% %\subsection{Geometrical Series}
% \begin{center}
%     \LARGE \textbf{Supplementary Material}\\
%     \large for ``Unexpected Effects of Online $k$-means Clustering"
% \end{center}
 
This appendix includes the proofs of the theorems and claims appearing throughout the paper. It consists of three sections: the first proves the main text's claims, the second technical claims on random samples, and the third general stand-alone technical claims.    

\section{Proofs of main theorems and claims}\label{apx:proofs_main_theorems}

\subsection{Random order, \texorpdfstring{$k=1$}{Lg}}
If the points' order is random, then there is a simple algorithm that uses only one center while preserving a constant approximation: simply taking the first point.
\begin{claim}\label{thm:k_1}
If the data points appear in a random order, there is an online algorithm that uses only one center and with probability at least $0.9$ it holds that $cost(alg)\leq 20\cdot cost(opt_1).$
\end{claim}
The main tool in proving the theorem, which will also be useful in cases where $k>1$, is the following known lemma that states that a random point in a cluster can be the good enough center for this cluster. 

\begin{lemma}\label{lemma:random_point_in_cluster}
Let $x_1,\ldots,x_n\in\mathbb{R}^d$ it holds that 
$$\E_{j\in[n]}\left[\sum_{i=1}^n\norm{x_i-x_j}^2\right]=2\sum_{i=1}^n\norm{x_i-\mu}^2,$$ where $\mu=\frac1n\sum_{i=1}^nx_i$ is the optimal center and $j$ is chosen uniformly at random from $[n]$.
\end{lemma}
For completeness, the proof of the lemma appears in Appendix~\ref{apx:auxiliary_claims}. To prove Lemma~\ref{thm:k_1} we use Lemma~\ref{lemma:random_point_in_cluster} and Markov's inequality. 

\subsection{Arbitrary order, \texorpdfstring{$k=1$}{Lg}}
We proceed to the case where the order of the data points is not random but it can appear in the worst order possible.
In this case we witness a surprising result --- it matters whether $n$ is known or not.   
%If $n$ is known in advance, then the following algorithm uses one center and has a constant approximation.  
% \begin{algorithm}
% \caption{On line clustering with $k=1$, $n$ known, arbitrary order}
% \begin{algorithmic} \label{alg:k_1_n_known}
% %\STATE \underline{data points appear in worst-case order, $n$ is known}
% \STATE pick $i\in[n]$ at random 
% \FOR {$t=1,\ldots,n$} 
% 	\IF {$t==i$} 
%         \STATE take $x_t$ as a center
%       \ENDIF
% \ENDFOR
% \end{algorithmic}
%\end{algorithm}
If $n$ is known in advance, then the algorithm can take one random point as a center and it will be a constant approximation algorithm.  %and %has a constant approximation. 

\begin{claim}\label{thm:k_1_worst_n_known}
There is an online algorithm that receives as input $n$, the size of the dataset and $n$ data points such that the following holds. For any order of the data points, the algorithm uses only one center and with probability at least $0.9$ it holds that $cost(alg)\leq 20\cdot cost(opt_1).$
\end{claim}
The correctness of the algorithm is similar to Claim~\ref{thm:k_1}. If $n$ is unknown in advance, one can use the doubling method. 

\begin{claim}\label{thm:k_1_worst_n_unknown_upper}
For any $c>1$ there is an algorithm that obtains $O(c)$-approximation with $O(\log_c n)$ centers, no matter what the order is and even if $n$ is unknown.
\end{claim}
Intuitively, since the algorithm does not know the value of $n$, it guess it and applies the algorithm from Theorem~\ref{thm:k_1_worst_n_known}. %Algorithm~\ref{alg:k_1_n_known}.
The algorithm starts by assuming that $n$ is small ($n=1$). Once more data is arrived, it increases the value of $n$ to $c$, and then to $c^2$ and so on. 
For each value of $n$, it applies the algorithm from Theorem~\ref{thm:k_1_worst_n_known}, %Algorithm~\ref{alg:k_1_n_known},
i.e., picks one random point in the next $n$ data points and take only this point as a center. 
This algorithm uses only $O(\log_c n)$ centers. Intuitively, it is $O(c)$-approximation because in the one before the last iteration most of points are still not received and the algorithm chooses a random point among many of them which yields a good center by Lemma~\ref{lemma:random_point_in_cluster}. The pseudocode appears in Algorithm~\ref{alg:k_1_n_unknown}.
%\begin{wrapfigure}{r}{0.5\textwidth}
 \begin{figure}[!htb]
\centering
\begin{minipage}{.5\textwidth}
\begin{algorithm}[H]\caption{Online clustering with $k=1$, $n$ unknown, arbitrary order}
% \caption{On line clustering with $k=1$, $n$ known, arbitrary order}
 \begin{algorithmic} \label{alg:k_1_n_unknown}
%\begin{algorithmic}
%\STATE \underline{data points appear in worst-case order, $n$ is unknown}
\STATE $last=0$, $n'=1$, $i^*=1$
\FOR {$t=1,\ldots$}
\IF {$t==last + i^*$} 
        \STATE take $x_t$ as a center
\ENDIF
\IF {$t == last + n'$ }
\STATE pick $i^*\in[n']$ at random 
\STATE $last = last + n'$, $n' := c \cdot n'$
\ENDIF
\ENDFOR
\end{algorithmic}
\end{algorithm}
\end{minipage}
\end{figure}
%\end{wrapfigure}

%\subsection*{Proof of Claim~\ref{thm:k_1_worst_n_unknown_upper}}
In the first line of Algorithm~\ref{alg:k_1_n_unknown}, it initialize the parameters. $last$ is the number of points we encountered before the current round, $n'$ is our current guess of the number of points, and $i^*$ is the point the algorithm will choose as center in the current round. In the loop, after the algorithm encounters $n'$ points, it increases $n'$ by a factor of $c$, we signal that we have encountered more examples by increasing $last$ and a new point, $i^*$ to be our next center is picked. 

\begin{proof}[of Claim~\ref{thm:k_1_worst_n_unknown_upper}]
It is easy to bound the number of centers Algorithm~\ref{alg:k_1_n_unknown} uses, it's bounded by the number of times $n'$ is increased plus $1$. The latter is bounded by $O(\log_c(n))$. 

Next we prove that the algorithm is an $O(c)$-approximation. 
For ease of notation, iteration $0$ is the first iteration and in this iteration $n'=c^0=1.$ 
The last iteration $i^*$ is the one where $\sum_{i=0}^{i^*}c^{i}>n$ for the first time. This means that 
\begin{equation}\label{eq:worst_n_unknown}
\frac{c^{i^*}-1}{c-1}\leq n < \frac{c^{i^*+1}-1}{c-1}.
\end{equation}
We now focus on iteration $i^*-1$ where the value of $n'$ is $c^{i^*-1}.$ 
 From Equation~\ref{eq:worst_n_unknown}, we know that $n\leq \frac{n'\cdot c^2-1}{c-1}\Rightarrow \frac{n(c-1)+1}{c^2}\leq n'\Rightarrow\frac{n}{2c}\leq n'$ (in the last equation we used that w.l.o.g $c\geq 2$). In words, in the one before the last iteration $n'$ is big compared to $n$, it's at least a fraction $1/2c$ of $n$.

From Lemma~\ref{lemma:random_point_in_cluster} and Markov's inequality we get that for at most $\frac{n}{20\cdot c}$ of the data points $x$ it holds that $cost(x) >40\cdot c\cdot cost(opt_1).$ Thus with probability at least $0.9$ the algorithm chooses a data point $x$ out of the $n'$ points at iteration $i^*-1$  with $cost(x)\leq 40c\cdot cost(opt_1).$
\end{proof}

%\subsection*{Proof of Theorem~\ref{thm:k_1_worst_n_unknown}}
%\paragraph{Proof of Theorem~\ref{thm:k_1_worst_n_unknown}}
\begin{proof}[%[proof
of Theorem~\ref{thm:k_1_worst_n_unknown}]
The examples are $1, 2, 3, \ldots$ and so on. 
Example with value $i$ appears $(7c)^{i}$ times (except the last value that might appear less than that).
%So there are $\floor{\log_{2c} (n(2c-1))-1}$ different values. 
 The points are given, group by group from smallest to largest.
Note that there are $\Omega(\log_c n)$ different points. 
We will show that any $c$-approximation algorithm must take $\Omega(\log_c n)$ examples as centers with probability at least $0.8$. This will finish the proof. 

There are two cases: either for each example (out of the $\Omega(\log_c n)$ different examples) the probability it will be taken as a center (at least one of the $(7c)^{t^*}$ instances of this example) is at least $0.9$ or not. 
First case implies, using Claim~\ref{clm:lower_bound_from large_center_to_large_randomness}, that with probability at least $0.8$ the algorithm takes as centers at least half of the examples, i.e., it takes $\Omega(\log_c n)$ examples as centers. 
Second case implies that there is an example $t^*$ such that probability that the algorithm will not take as center is more than $0.1.$ 
We will show that the algorithm is not a $c$-approximation and we will reach a contradiction. 

%Focus on the first example $t^*$ such that probability that the algorithm will not take all examples $1,\ldots, t^*$ is smaller than $0.9.$
%Assume by contradiction that $t^*$ is not $\Omega(\log_c n)$ and we will show that the algorithm is not a $c$-approximation. 
%Once the algorithm did not choose example $t^*$ as a center (even though there are $(7c)^{t^*}$ instances of this example) 
We focus on the series $1,\ldots, t^*$. %, the examples stop. 
Note that since $n$ is unknown the algorithm works the same up until $t^*$, no matter if there are more examples after $t^*$ or not. 
Also recall that the algorithm should work correctly for any dataset.
This means that also for the shorter series that includes only examples $1,\ldots, t^*$, the probability that $t^*$ is taken as a center is smaller than $0.1.$
We will prove that if $t^*$ is not taken, which happens with probability more than $0.1$, it holds that $c\cdot cost(opt_1) < cost(alg)$ for the dataset that includes the points in $1,\ldots,t^*$, which is a contradiction to the assumption that the algorithm is a $c$-approximation (i.e., that $cost(alg)\leq c\cdot cost(opt_1)$ with probability at least $0.9$).

%be should work for any series and specifically on the latter. %that it is valid to stop, as $n$ is unknown.  
%Otherwise the examples continues until the first iteration where $\sum_{i=1} (7c)^i>n.$

%Number of different examples is bounded by $\ceil{\log_{7c}(n)}$. If the learner chooses all examples, the theorem follows.  If the learner did not take a point $t^*$, then we will show that the learner didn't find a $c$-approximation, which is a contradiction.
  
 The cost of the algorithm is at least $cost(alg)\geq (7c)^{t^*},$ as there are $(7c)^{t^*}$ examples with value $t^*$ and the distance to the closest center is at least $1.$ Let's compare it to the cost of the optimal solution. 
One can take example $x_{t^*}$ and the optimal cost is only smaller, i.e., $cost(opt_1)\leq\sum_{i=1}^{t^*}(7c)^i({t^*}-i)^2.$ 
Using Claim~\ref{clm:geometric_series_squared_diff_bound} we have that 
$$cost(opt_1)\leq 6(7c)^{t^*-1}$$
This implies that $$c\cdot cost(opt_1) \leq 6c (7c)^{t^*-1} < (7c)^{t^*}\leq cost(alg),$$
which is a contradiction to the $c$-approximation of the algorithm. 
\end{proof}

%\subsection*{Proof of Theorem~\ref{thm_k_2_worst}}
\subsection{Arbitrary order, \texorpdfstring{$k\geq 2$}{Lg}}
If the order of the dataset can be arbitrary, then any $c$-approximation algorithm must basically take all the points in the dataset. 
The idea of the proof is the following. We present $n$ non-negative points on the line, i.e., $x_t\in\mathbb{R}$, in increasing order. 
Each point is much further than the previous one. Since it is so further away it has to be taken, otherwise, the rest of the points will be set to $0$, which is allowed as points arrive in an arbitrary order. In this case the largest point has to be taken or the  ratio $\frac{cost(alg)}{cost(opt_2)}$ is arbitrarily large. 
Thus, all the points need to be taken if the algorithm is deterministic. If the algorithm is stochastic $\Omega(n)$ of the points need to be taken with probability at least $0.8.$
%The details appear in the appendix. 

Any algorithm can take $n$ data points as centers and achieve the minimal cost of $0$. 
Thus, we conclude that the upper and lower bounds coincide, up to a constant, when the order is arbitrary. 

\begin{claim}\label{thm_k_2_worst}
For any integers $k\geq2$ and $n$, any scalar $c>1$, and for any clustering algorithm that is a $c$-approximation (even if $n$ is known) there are $n$ points and an ordering of them such that the algorithm must take $\Omega(n)$ centers with probability at least $0.8$.
\end{claim}
\begin{proof}%[proof of Theorem~\ref{thm_k_2_worst}]
We will define a series of points in $\mathbb{R}$, $0=x_1 < x_2 < \ldots < x_n$ such that any $c$-approximation algorithm must take $\Omega(n)$ points as centers with probability at least $0.8$. This will finish the proof. 

If the probability that each example is taken as center is at least $0.9$, then with probability at least $0.8$ the algorithm takes $0.5n=\Omega(n)$ points as centers, see Claim~\ref{clm:lower_bound_from large_center_to_large_randomness}, and the theorem follows. 
Focus on the first example $x_{t^*}$ such that the probability that the algorithm will not take it as a center is more than $0.1.$
%Assume by contradiction that $t^*<n$ and 
We will show that the algorithm is not a $c$-approximation. 
We focus on the series of points $x_{1},\ldots, x_{t^*}, 0, 0, \ldots, 0$. 
Recall that for the algorithm to be a $c$-approximation, it should function correctly under any dataset and specifically one the latter. 
Note that up until example $x_{t^*}$ the two series ($(x_1,x_2,\ldots, x_n)$ and $(x_{1},\ldots, x_{t^*}, 0, 0, \ldots, 0)$) are the same, thus the probability that $x_{t^*}$ will not be taken as a center is more than $0.1.$
%We will show that the learner will have to take each point in this series for the algorithm to be a $c$-approximation. 
%Once the algorithm decides not to take one of the data points $x_{t^*}$ as a center then for any $t>t^*$ we set $x_{t}=0$.
%If the learner takes all points, then we are done. Else, we will show that the cost is arbitrarily large and we are also done (it will contradict the assumption that the algorithm is a $c$-approximation).  
%Define $x_1=0$. 

One can take $x_{t^*}$ and $0$ as centers, the optimal cost can only be smaller, thus $$cost(opt_k)\leq \sum_{t=2}^{t^*-1}\norm{x_t-x_1}^2=\sum_{t=2}^{t^*-1}x_t^2.$$
Since the algorithm did not take $x_{t^*}$ as a center we have that $$cost(alg)\geq \norm{x_{t^*}-x_{t^*-1}}^2.$$
For this lower bound to work, we take the series of examples such the following strict inequality  holds  $$cost(alg) \geq (x_{t^*}-x_{t^*-1})^2>c\cdot \sum_{t=2}^{t^*-1}x_t^2\geq c\cdot cost(opt_k).$$
In different words, 
$$x_{t^*} > x_{t^*-1} + \sqrt{c\cdot \sum_{t=2}^{t^*-1}x_t^2}$$
And we are done since with probability more than $0.1$ it holds that $cost(alg)>c\cdot cost(opt_k)$, i.e., the algorithm is not a $c$-approximation. 
\end{proof}

\subsection{Random order, \texorpdfstring{$k\geq 2$}{Lg}, known \texorpdfstring{$n$}{Lg}} \label{apx:k_2_random_order_known_n}

To make our claims more general, we consider an arbitrary cost that its distance function satisfy a version of a triangle inequality: there is a \emph{constant} $D$ such that  
\begin{equation}%\label{eq:generalized_triangle_inequality}
\forall u,v,w.\quad d(u,v)\leq D\cdot(d(u,w)+d(w,v)).
\end{equation}
We call such a cost a \emph{$D$-cost}.
 
 %\paragraph{Proof of Theorem~\ref{thm:k_2_upper_random_known_n}}

\begin{proof}[%proof 
of Theorem~\ref{thm:k_2_upper_random_known_n}]
We prove the claim under the general cost function, see Section~\ref{sec:general_cost_function}. For the $k$-means cost $\dis{x}{y}=\norm{x-y}^2$,  Inequality~\ref{eq:generalized_triangle_inequality} holds with $D=2.$ To ease the analysis, we add an initial step to the algorithm and take the first point as a center. This will increase the number of centers by only $1.$

A few notation first. Number of examples is $n$, number of examples the algorithm saves in the first phase is equal to $|M_1|=\alpha n$, where $\alpha := \frac{\floor{\frac{0.01}k n}}{n}$, and $|M_2|=\floor{\alpha_2 n}$ is the number of points received in phase $2,$ where $\alpha_2=\nicefrac{\alpha}{10^5k^3}.$ 

In the proof we consider three clusterings of the entire dataset. 
The first, $C^{M_1}=(C^{M_1}_i)_{i=1}^k$, is the optimal clustering with respect to the points $M_1$ the algorithm saves in the first phase.  
The second is the optimal clustering $C^*=(C^*_i)_{i=1}^k$ induced by the entire dataset.
And the third $C=(C_i)_{i=1}$ is induced by all the centers taken by the algorithm in the third phase.  
We prove that with probability at least $0.9$ the following two claims hold for any constant $k$
\begin{enumerate}
\item The number of centers the algorithm takes is $\Theta(1)$ %$O(k^3)=\Theta(1)$, for constant $k$. 
\item %$cost(C) \leq 8^k(k!)^3 \cdot cost(C^*),$ i.e., for constant $k$: 
$cost(C) \leq \Theta(1) \cdot cost(C^*).$
\end{enumerate}
%\underline{Bounding number of centers:}Let us start with the first argument, bounding the number of centers the algorithm takes. \ref{clm:k_2_upper_known_n_bound_centers}
This is proved in Claims~\ref{clm:k_2_upper_known_n_bound_centers} and \ref{clm:k_2_upper_known_n_bound_cost}.
\end{proof}

\begin{claim}\label{clm:k_2_upper_known_n_bound_centers}
 Algorithm~\ref{alg:k_2_upper_known_n} takes as center at most $O(k^5)$ points, with probability at least $0.98.$
\end{claim}
\begin{proof}
The algorithm takes as a center two types of points. Either close (Line~\ref{alg:k_2_upper_known_n_line_take_close_point} in Algorithm~\ref{alg:k_2_upper_known_n}) or far (Line~\ref{alg:k_2_upper_known_n_line_take_far_point}). Bounding the number of close points by $O(k^2\log k)$ is easy, as it follows immediately from the definition of the algorithm --- the algorithm only takes $O(k\log k)$ close points per cluster. The interesting claim is bounding the number of centers that are far.

The key idea is that the points in the second phase, $M_2$, that constitute a good representation of the  clustering created after the first phase, $C^{M_1}$, in the sense that the algorithm receive a linear number of points, $\beta|C_i|$, from each large enough cluster $C^{M_1}_i\in C^{M_1}$ (this is true using Claim~\ref{clm:random_from_each_part}). A point $x\in C^{M_1}_i$ is considered \emph{far} from the cluster center, $c_i\in C^{M_1}_i$, if $x$ is furthest from all the $\beta|C^{M_1}_i|$ points received in the second phase. The probability that $x$ will be far is very small, only $\nicefrac1{(\beta|C^{M_1}_i|)}.$ Thus, the total number of centers the algorithm takes is only a constant. 
For small clusters, the algorithm can take all the points in the cluster, and still not take too many centers. 

More formally, for each cluster $C^{M_1}_i\in C^{M_1}$ we separate the analysis depending on whether the points left from the cluster in is small or large. 
If it's small, i.e., $|C^{M_1}_i-M_1|\leq\frac{400k}{\alpha_2}$, then the algorithm takes at most all the points remaining in this cluster, $\frac{400k}{\alpha_2}=O(k^4)$. Since there are at most $k$ small clusters (because there are $k$ clusters in total), the total number of centers taken because of small clusters is only $O(k^5).$

For large clusters, $|C^{M_1}_i-M_1|>\frac{400k}{\alpha_2}$, Claim~\ref{clm:random_from_each_part} and union bound prove that, with probability at least $0.99$, for all large clusters, 
many points are received in the second phase, i.e.,
\begin{eqnarray}\label{eq:random_k_known_n_large_from_each_cluster}
%\frac{\alpha}2|C^{M_1}_i|\leq |C^{M_1}_i \cap (M_1\cup M_2)|. %\leq 2\alpha|C^M_i|
\frac{\alpha_2}2|C^{M_1}_i-M_1|\leq |(C^{M_1}_i -M_1)\cap M_2|. %\leq 2\alpha|C^M_i|
\end{eqnarray}

A point $x\in C^{M_1}_i$ is taken as a far point if it's far from the center of the cluster $c^{M_1}_i$ compared to all points in this cluster received in phase 2, i.e, if $\dis{x}{c_i^{M_1}}>\dis{y}{c^{M_1}_i}$ for all $y\in C^{M_1}_i$ received in phase 2. 
%it's further away from all points in this cluster compared to points received in the first and second phase. In different words, a point is taken as a far center if it is farthest away from $c^{M_1}_i$ compared to all points in $C^{M_1}_i \cap (M_1\cup M_2)$. 
From Inequality~\ref{eq:random_k_known_n_large_from_each_cluster}, the probability of taking a point as far, happens with probability at most $$\frac2{{\alpha_2}|C^{M_1}_i-M_1|}.$$
We can also deduce from Inequality~\ref{eq:random_k_known_n_large_from_each_cluster} that there are at most $(1-\frac{\alpha_2}2)|C^{M_1}_i-M_1|$ points in $C^{M_1}_i$ that are received in the third phase. %(i.e., that there are not in $M$).
So in total, the expected number of points for cluster $C^{M_1}_i$ taken as far points is bounded by 
$$\frac2{\alpha_2|C^{M_1}_i-M_1|}\cdot\left(1-\frac{\alpha_2}2\right)|C^{M_1}_i-M_1|=\frac{2-\alpha_2}{\alpha_2}.$$ 
Use Markov's inequality to show that with probability at least $0.99$ all the large clusters cause at most $O(\frac{2-\alpha_2}{\alpha_2}\cdot k^2)=O(k^5)$ far centers. 
Summing the two cases, the number of points taken as centers by the algorithm is $O(k^5)$ with probability at least $0.98.$ 
\end{proof}

\begin{claim}\label{clm:k_2_upper_known_n_bound_cost}
 The cost of Algorithm~\ref{alg:k_2_upper_known_n} is bounded by $\Theta(1)\cdot cost(opt_k),$ with probability at least $0.95.$
\end{claim}
\begin{proof}
Denote by $C$ the clustering returned by Algorithm~\ref{alg:k_2_upper_known_n}. 
We want to prove that $cost(C)$ is at most some function of $k$ times $cost(opt_{k}).$
We will prove a stronger result. 
We will prove that when the algorithm is given a dataset and a parameter $k$ it performs well compared to any $opt_{k'}$ for any integer $k'\leq k$. Namely, we prove that for any integer $k'\leq k$ with probability at least $1-\frac{5k'}{100k}$
%$$cost(alg) \leq 3\cdot8^k(k!)^3 cost(opt_k).$$
$$cost(C) \leq\left(\frac{26kD^5a}{\alpha}\right)^{k'}\cdot cost(opt_{k'})$$
Fix $k$. We prove this claim by induction on $k'.$ 
For $k'=1$, recall that to simplify the analysis the algorithm takes the first point as center, and thus the claim follows immediately, similarly to Claim~\ref{thm:k_1}. %because $20\leq 3\cdot 8^1\cdot(1!)^3$. 
 %If it small, we will show the algorithm takes a member from $C^*_i$ and it is good enough. 
%For $k'>1$, we go over all clusters in the optimal clustering $(C^*_1,\ldots, C^*_{k'})$ and separate the analysis into two cases depending on whether $C^*_i$ is large or not.

For $k'>1$, we use Claims~\ref{clm:k_2_known_n_small} and \ref{clm:k_2_known_n_large} with $\delta=\frac{5}{100k}$. Specifically, if there is a cluster such that $cost(opt_{k'-1})\leq\frac{26kD^5a}{\alpha}cost(opt_{k'})$ (it's case 2 in both of the claims) then using the induction hypothesis $$cost(C)\leq \left(\frac{26kD^5a}{\alpha}\right)^{k'-1}\cdot cost(opt_{k'-1})\leq \left(\frac{10^7D^5k^5}{\alpha^2}\right)^{k'} cost(opt_{k'}).$$
Otherwise, these claims prove that $$cost(C)\leq 5D^2k \cdot cost(opt_{k'})\leq \left(\frac{10^7D^5k^5}{\alpha^2}\right)^{k'} cost(opt_{k'}).$$
\end{proof}

In the following claim we show that for a small optimal cluster (smaller than $poly(k)$) either a good center will be taken for it or it can be merged into another cluster.
\begin{claim}\label{clm:k_2_known_n_small}
Fix $\delta\in(0,1)$ and a $D$-cost.
Suppose Algorithm~\ref{alg:k_2_upper_known_n} has the following parameters.
\begin{itemize}
    \item First phase: use $\alpha n$ points and an $a$-approximation algorithm with $k\geq k'$ clusters
    \item Second phase: use $\beta n$ points with $\beta\leq \frac{\alpha\delta^2}{32k}$ and $\alpha+\beta\leq\frac\delta2$
\end{itemize}
For any optimal clusters $C^*_i$, $i\in[k']$, such that $|C^*_i|< \frac{16}{\alpha\delta}$,
 with probability at least $1-\delta$
one of the following holds:
\begin{enumerate}
    \item a point $v\in C^*_i$ will be taken as a center with $\sum_{x\in C^*_i}d(x,v)\leq 2D\cdot cost(opt_{k'})$ 
    \item $cost(opt_{k'-1})\leq\frac{13kD^5a}{\alpha}\cdot cost(opt_{k'})$
\end{enumerate}
\end{claim}

\begin{proof}
Focus on the closest point $x_i\in C^*_i$ to the center $c^*_i.$
The cost of taking $x_i$ as a center to $C^*_i$ is 
\begin{equation}\label{eq:cost_closest_point_as_center_1}
\sum_{x\in C^*_i}d(x,x_i)\leq D\sum_{x\in C^*_i}d(x,c^*_i) + D\sum_{x\in C^*_i}d(c^*_i,x_i)\leq 2D\sum_{x\in C^*_i}d(x,c^*_i),
\end{equation}
in the first inequality we used Inequality~\ref{eq:generalized_triangle_inequality} and in the second we used the fact that $x_i$ is closest to $c^*_i$ than all points in $C^*_i.$

The probability that $x_i$ will be received in phase $1$ or $2$ is $\alpha+\beta\leq\delta/2.$ The rest of the analysis will focus on the case that $x_i$ was received only on phase $3$.
The point $x_i$ is in some cluster $C^{M_1}_i\in C^{M_1}$ with center $c^{M_1}_i.$ 
Focus on all points in $C^{M_1}_i$ that are furthest from the center $c^{M_1}_i$ than $x_i:$
$$A=\left\{x\in C^{M_1}_i \wedge x\neq x_i: d(x,c^{M_1}_i)\geq  d(x_i,c^{M_1}_i)\right\}.$$

We separate the analysis into two cases depending on the the size of $A.$ If $|A|\leq \frac\delta{2\beta}.$ Then, using Claim~\ref{clm:small_random_set}, we know that with probability at least $1-\frac{\delta}2$ the algorithm will not get a member of $|A|$ in the second phase, and thus a $x_i$ will be taken as center at phase $3$, which completes the proof of case $1$ in the claim.  

If $|A|>\frac\delta{2\beta}\geq\frac{16k}{\alpha\delta}$, then there is an optimal cluster $C^*_j$ with center $c^*_j$  that includes at least $\frac{16}{\alpha\delta}$ of the points in $A$. We will show that $C^*_i$ can be merged into $C^*_j.$ 
We want to bound the cost of the following clustering with ${k'}-1$ centers: $c^*_1,\ldots, c^*_{k'}$ without $c^*_i$ and all points in $C^*_i$ will be assigned to $c^*_j$. The cost is equal to 
\begin{eqnarray*}
\sum_{x\in C^*_i} \dis{x}{c^*_j} + \sum_{r\neq i} \sum_{x\in C^*_r} \dis{x}{c^*_r}
\end{eqnarray*}
Let us bound the first sum using Inequality~\ref{eq:general_cost} twice 
\begin{eqnarray*}
\sum_{x\in C^*_i}\dis{x}{c^*_j} &\leq& D\sum_{x\in C^*_i}\dis{x}{x_i} + D^2\sum_{x\in C^*_i}\dis{x_i}{c^{M_1}_i}+D^2\sum_{x\in C^*_i}\dis{c^{M_1}_i}{c^*_j}\\
&\leq& 2D^2\sum_{x\in C^*_i}\dis{x}{c^*_i}+D^2\sum_{x\in A\cap C^*_j}d(x,c^{M_1}_i)+D^2|A\cap C^*_j|d(c^{M_1}_i,c^*_j),
%\\&\leq& D\cdot cost(opt_{k})+\frac{6kD^4a}{\alpha}cost(opt_{k'})\\&\leq&\frac{7kD^4a}{\alpha}cost(opt_{k'})
\end{eqnarray*}
where in the first inequality: the first term is bounded by Inequality~\ref{eq:cost_closest_point_as_center_1}; the second term by the fact that for each member $x\in A$ have $d(x,c^{M_1}_i)\geq d(x_i,c^{M_1}_i)$ and  $|A\cap C^*_j|\geq\frac{16}{\alpha\delta}\geq|C^*_i|$; for the third term we used again $|A\cap C^*_j|\geq|C^*_i|.$

To bound the last expression we use Claim~\ref{clm:known_bound_cost_large_cluster_contains_many_points_good_cluster} and Claim~\ref{clm:known_n_bound_cost_large_cluster_close_to_center} to deduce that with probability at least  $1-\delta$, $cost(opt_{k'-1})$ is bounded by 
$$2D^2 cost(opt_{k'}) + \frac{5kD^4a}{\alpha}cost(opt_{k'})+ \frac{6kD^5a}{\alpha}cost(opt_{k'})\leq \frac{13kD^5a}{\alpha}cost(opt_{k'}).$$
\end{proof}

In the following claim we show that for a large optimal cluster, either a good center will be taken for it or it can be merged into another cluster. 
\begin{claim}\label{clm:k_2_known_n_large}
Fix $\delta\in(0,1)$ and $D$-cost.
Suppose Algorithm~\ref{alg:k_2_upper_known_n} has the following parameters.  
\begin{itemize}
    \item First phase: use $\alpha n$ points and an $a$-approximation algorithm with $k\geq k'$ clusters
    \item Second phase: use $\beta n$ points with $\alpha+\beta\leq\frac1{4k}.$
    \item Third phase: at least $3k\log\frac2\delta$ are taken as centers as close points, for each cluster 
\end{itemize}
For any optimal clusters $C^*_i$, $i\in[k']$, such that $|C^*_i|\geq \frac{16}{\alpha\delta}$,
 with probability at least $1-\delta$
one of the following holds:
\begin{enumerate}
    \item a point $v\in C^*_i$ will be taken as a center with that $\sum_{x\in C^*_i}d(x,v)\leq 5D^2\cdot cost(opt_{k'})$ 
    \item $cost(opt_{k'-1})\leq\frac{26kD^5a}{\alpha}\cdot cost(opt_{k'})$
\end{enumerate}
\end{claim}
\begin{proof}
Denote by $G_i$ the $\frac{|C^*_i|}{2}$ closest points in $C^*_i$ to the center $c^*_i$. The cost of taking any $g\in G_i$ as a center to $C^*_i$ is small 
\begin{equation}%\label{eq:cost_closest_point_as_center}
\sum_{x\in C^*_i}d(x,g)\leq D\sum_{x\in C^*_i}d(x,c^*_i) + D\sum_{x\in C^*_i}d(g,c^*_i)\leq 5D^2\sum_{x\in C^*_i}d(x,c^*_i),
\end{equation}
where in the first inequality we used Inequality~\ref{eq:generalized_triangle_inequality} and in the second we used the definition of $G_i$, Lemma~\ref{lemma:general_cost_expected_cost}, and Markov's inequality. 
We will show that either the algorithm takes a point from $G_i$ as a center in the third phase or $C^*_i$ can be merged into another optimal clustering without harming the cost by much. 

There is a cluster in $C^{M_1}_{i}\in C^{M_1}$ that contains at least $|G_i|/k$ of the points in $C^*_i.$
There are two cases: either $C^{M_1}_{i}$ contains at least $|G_i|/k$ points from a different optimal cluster $C^*_j$ or not. 
If it contains that many points from a different cluster $C^*_j$, then, by Claim~\ref{clm:known_n_bound_cost_large_cluster_two_cluster_are_close}, with probability at least $1-\frac{\delta}2$ it holds that 
$$cost(opt_{k'-1})\leq\frac{26k^2D^5a}{\alpha}\cdot cost(opt_{k'}),$$
which completes the proof of case $2$ in the claim. 
Otherwise, %there are at most $|G_i|$ points in $C^{M_1}_i$ that are not in $G_i$.
from Claim~\ref{clm:random_from_each_part}, with probability at least $1-\delta/2$ at most $2(\alpha+\beta)|G_i|$ points from $G_i$ are received in phase $1$ and $2$. Or, in different words at least $(\frac{1}{k} - 2\alpha-2\beta)|G_i|\geq\frac{1}{2k}{|G_i|}$ points in $G_i\cap C^{M_1}_i$ are received in phase $3.$ 

In the current case, $C^{M_1}_i$ contains at most $|G_i|$ points from clusters different from $C^*_i$. Thus, from Corollary~\ref{cor:random_sample_enough_points_hit_set}
the first $3k\log\frac{2}{\delta}$ points in $C^{M_1}_i$ in third phase with probability $1-\delta/2$ is a member of $G_i$ and this completes case $1$ in the claim. 
\end{proof}

\paragraph{Comparison between Algorithm~\ref{alg:k_2_upper_known_n} and \cite{indyk99}:}
The work \cite{indyk99} designed a sublinear time algorithm for $k$-medians, which is similar to the $k$-means problem discussed in this paper. Their algorithm 
 has some similarities to Algorithm~\ref{alg:k_2_upper_known_n}, that works in the case that number of points in the dataset, $n$, is known in advance and the order of the dataset is random. 
Several differences cause the analysis and algorithm to be different: 
\begin{enumerate}
 \item Parameter regime: ``large'' cluster in inherently different in the two algorithms. 
 In \cite{indyk99}, large means $O(\sqrt{n})$, as they cannot take more points for the algorithm to be with sublinear time.  On the other hand, for Algorithm~\ref{alg:k_2_upper_known_n} ``large'' means some constant fraction because the algorithm is allowed to take only a constant number of centers in the second phase. 
\item  Centers from phase 1:  \cite{indyk99} simply takes the centers that were chosen in phase $1$. In our framework this is not allowed since once a center was observed the algorithm cannot retake it. 
To overcome this obstacle we take centers that are close to the centers chosen in phase $1$. But then we need to decide what is close which complicates the analysis. 
\item Far points: in 
\cite{indyk99}, far points are the furthest points from the cluster defined in phase $1$. We cannot use this definition in our framework. To resolve this issue, we add an intermediate step where the algorithm saves a constant fraction of  number of points to set a bar that defines far. In the last phase, only points that are above the bar, are taken as centers. Note that the points in the first phase cannot be used to define the bar, as they were used to define the cluster. 
%This is why we need to choose centers that are close to the one chosen in the first phase. 
% \item Different analysis and algorithm: the outcome of all the above differences is that our correctness proof is distinct from the one in \cite{indyk99} and that we added an additional step besides the two phases used in \cite{indyk99}. 
 \end{enumerate}

\subsection{Random order, \texorpdfstring{$k\geq 2$}{Lg}, unknown \texorpdfstring{$n$}{Lg}}\label{apx:k_2_random_order_unknown_n}
%\subsection*{Proof of Theorem~\ref{thm:k_2_upper_random}}

\subsubsection{Lower bounds}

\begin{claim}\label{thm:k_2_lower_random}
For any integer $n$, any scalar $c>1$, and for any clustering algorithm that does not know what $n$ is and is a $c$-approximation, there are $n$ points such that the algorithm must take $\Omega(\log n)$ centers with probability at least $0.7$.
\end{claim}
%The idea of the proof is the use the same data  points as in Theorem~\ref{thm_k_2_worst} since it's random order it will give a weaker bound (EXPLAIN intuitively why). 
To prove the theorem we take the same dataset as in the proof of Claim~\ref{thm_k_2_worst}. In this construction, at each iteration the point with the maximal value has to be taken, otherwise the examples can stop and the algorithm will not be a $c$-approximation. If the order is random, there will be $\Omega(\log n)$ points that are maximal, as the probability that the $i$-th point to be maximal is about $1/i$. Thus $\Omega(\log n)$ have to be taken as centers. 

%\paragraph{Proof of Claim~\ref{thm:k_2_lower_random}}
\begin{proof}[%proof 
of Claim~\ref{thm:k_2_lower_random}]
We will use the same series of points in $\reals$ that was used in the proof of Theorem~\ref{thm_k_2_worst}.
We say that a point in the $i$-th iteration is \emph{maximal} if it's the largest value so far.  
We prove the following two claims:
\begin{enumerate}
\item For any $c$-approximation it must take at least $0.5$ of the maximal points with probability at least $0.8$. 
\item With probability at least $0.99$ there are $\Omega(\log n)$ maximal points. 
\end{enumerate}
Once we prove the two steps we are done. 

Claim 1 -  There are two cases (i) for each point if it's a maximal point, the probability the algorithm takes it as a center is at least $0.9$ (ii) there is a point $x$ that if $x$ is a maximal point the probability the algorithm takes it as a center is less than $0.9$. 
In case (i), using Claim~\ref{clm:lower_bound_from large_center_to_large_randomness}, we know that with probability at least $0.8$ the algorithm takes half of the maximal points. In case (ii), from the same argument as the proof of Theorem~\ref{thm_k_2_worst} we have a contradiction to the assumption that the algorithm is a $c$-approximation using the dataset $1,\ldots,x.$

%Claim 1 - assume by contradiction that there is an iteration such that a maximal point is not chosen as a center. 
%Then the examples will stop (this is valid as $n$ is unknown). In this case, from the same argument as the proof of Theorem~\ref{thm_k_2_worst} we have a contradiction to the assumption that the learner is a $c$-approximation.
%%There are two cases either there is a maximal point that the probability to take it as a center is smaller than $0.9$,  

Claim 2 - denote by $X$ the random variable that is equal to the number of maximal points and denote by $X_i$ the binary random variable that is equal to $1$ if the $i$-th example is maximal, and otherwise $X_i=0.$  
For any $i$, $\E[X_i]$ is equal to the probability that the $i$-th example is the largest than all previous examples. 
This probability is equal to $\E[X_i]=1/i$.
%Also note that the $X_i$'s are independent.  
Thus  $$\E[X]=\E\left[\sum_{i=1}^n X_i\right] = \sum_{i=1}^n \frac{1}{i} \leq \log n +1.$$
Note that $X_i$'s are independent and thus we can use Hoeffding's inequality.
\end{proof}

For the case of random order, unknown $n$ and general $k$, we will show a lower bound of $\Omega(k\log\frac{n}k)$ on the number of centers and a matching upper bound.

%\paragraph{Proof of Theorem~\ref{thm:k_lower_random}}
\begin{proof}[%proof 
of Theorem~\ref{thm:k_lower_random}]
The high-level idea is similar to the one in Claim~\ref{thm:k_2_lower_random}. We will have a series of numbers $w_1 < \ldots < w_n$ with increasing distances $w_{i+1}-w_i$. We will show that since $n$ is unknown, the $k-1$ largest number at each step have to be taken as centers for the algorithm to be a $c$-approximation algorithm, and there are $\Omega\left(k\log\frac n k\right)$ points that are at some step among the $(k-1)$ largest numbers. 

%We will take $k$ series of points in $\reals^k$. The $i$-th series will contain $\floor{\frac{n}{k}}$ points each will be zeros everywhere except index $i$. Projecting them into the $i$-th coordinate will form exactly the series described in Theorem~\ref{thm_k_2_worst}.

We say that the point received in the $t$-th iteration is \emph{$k$-maximal} if it is one of the $k-1$ largest values among the points received till iteration $t.$
We choose our dataset such that at each iteration the $k-1$ largest points so far must be taken as centers. Namely, the $i+1$ point in the dataset is chosen such that 
$$(w_{i+1}-w_i)^2 > c\sum_{j=1}^i(w_i-w_j)^2.$$
Suppose an algorithm does not take a $k$-maximal point at time $t$, and denote by $w_{i}$ the $k$-th largest point at this time step. 
The algorithm's cost, $cost(alg)$, is at least $(w_{i+1}-w_i)^2$. On the other hand, one can take the $k$ largest points as centers and get a bound on the optimal cost $cost(opt)\leq \sum_{j=1}^i(w_i-w_j)^2.$  By the choice of the dataset we get that if a $k$-maximal point is not taken as a center, than $cost(alg)>c\cdot cost(opt).$

Expected number of points that are $k$-maximal is 
$$k-2+\sum_{j=k-1}^n\frac{k-1}{j}=\Omega\left(k\log \frac{n}{k}\right).$$

The theorem follows from the following two claims:
\begin{enumerate}
\item For any $c$-approximation it must take at least $0.5$ of the $k$-maximal points with probability at least $0.8$. 
\item With probability at least $0.99$ there are $\Omega(k\log \frac{n}{k})$ maximal points.

% We say that a point in the $i$-th iteration is \emph{maximal} if it's the largest value so far comparing to the points with a similar non-zero coordinate.  
% There theorem follows from the following two claims:
% \begin{enumerate}
% \item For any $c$-approximation it must take at least $0.5$ of the maximal points with probability at least $0.8$. 
% \item With probability at least $0.99$ there are $\Omega(k\log \frac{n}{k})$ maximal points. 
% %\item If a point is maximal, the learner must choose it as a center.  
% %\item With probability at least $0.9$ there are $\Omega(k\log \frac{n}{k})$ maximal points. 
\end{enumerate}
The two claims follow from similar arguments as the proof of Claim~\ref{thm:k_2_lower_random}.
% More specifically, apply Theorem~\ref{thm:k_2_lower_random} for each group, which implies that there are $\Omega(\log \frac{n}k)$ maximal for each of the $k$ coordinates. 
\end{proof}

\subsubsection{Upper bounds}
%\paragraph{Proof of Theorem~\ref{thm:k_2_upper_random}}

\begin{proof}[%proof 
of Theorem~\ref{thm:k_2_upper_random}]
We start by bounding the number of expected centers the algorithm uses. 
Denote by $X$ the random variable that is equal to the number of centers chosen. 
Denote by $X_i$ the random variable that is $1$ if the $i$-th point is taken as center and $0$ otherwise. 
Then, the expected number of centers the algorithm chooses is equal to $$\E[X]=\E[\sum_{i=1}^n X_i]=\sum_{i=1}^n\E[X_i]=\sum_{i=1}^n \Pr(i\text{-th point is a center})$$
By the definition of the algorithm $X_1$ is always $1$. 
For $1<i\leq n$ it holds that $X_i=1$ only if the $i$-th point is the furthest away from the first point among points $1,\ldots,i-1$. The probability that the $i$-th point is the furthest is equal to $\frac{1}{i-1}$. Thus $$\E[X]=1+ \sum_{i=2}^{n} \frac1{i-1} = 1+ \sum_{i=1}^{n-1} \frac1 i \leq  \log(n-1) +2.$$ 
From Markov's inequality with probability at least $0.99$ number of centers is $O(\log n)$. %at most $$20(1 + \log(n-1))=O(\log n).$$

We are now left with proving that Algorithm~\ref{alg:k_2} is a $\Theta(1)$-approximation.
Focus on the optimal clustering. 
Denote by $C_1^*$ the points in the first cluster and by $C_2^*$ the points in the second cluster. 
We define the set of \emph{good} points for a cluster $r$ ($r=1,2$) as the set of points that taking them as a center will not increase the cost by much. More formally, 
$$Good_r = \{y_r \in C^*_r | \sum_{x_i \in C^*_r} \norm{x_i-y_r}^2 \leq 100 \sum_{x_i \in C^*_r} \norm{x_i-\mu^*(r)}^2\},$$
where  $\mu^*(r)=\frac{1}{|C^*_r|}\sum_{x_i \in C^*_r} x_i$ is the optimal center for cluster $C^*_r$.
From Lemma~\ref{lemma:random_point_in_cluster} and Markov's inequality we know that $Good_r$ is large. 
Specifically, $$\frac{|Good_r|}{|C^*_r|}\geq1-\frac{1}{50}.$$
Thus, using union bound, with probability at least $1-\frac{2}{50}$ the first point the algorithm encounters from each cluster $r$ is good.

Fix the first point the algorithm encounters by $x$, w.l.o.g $x\in C^*_1$. 
Define by $y_2^*\in Good_2$ the closest point to $x$ in $Good_2$, i.e.,
$$y_2^* = \argmin_{y_2\in Good_2}\norm{y_2-x}.$$ 
Denote by $B$ all the points in $C^*_1$ that are further from $x$ than $y^*_2$, i.e., 
$$B = \{y_1\in C^*_1 : \norm{y_1-x} \geq \norm{y_2^*-x} \}.$$
%In different words, the first point is a good representative of the first cluster. 
%Next we want to show that a point $y_2\in Good_2$ will be chosen. 
%the first point $y_2$ the algorithm encounters from the second cluster will be chosen. After proving that we are done. 
There are two cases 
\begin{enumerate}
\item $|B|\leq 0.01|C_2^*|$: we will show that most probably, the first point the algorithm encounters in $C^*_2$ will be chosen as a center. We know that $$|B| \leq 0.01|C_2^*| \leq 0.02 |Good_2|.$$ 
Thus, with probability at least $1-0.02$, the first point in $Good_2\cup B$ is in $Good_2$ and the algorithm takes it as a center. 

%If a point from $Good_2$ appears before all points from $B$, then the first point the algorithm encounters from cluster $C^*_2$ will be chosen as a center.
%This probability is 
\item $|B| > 0.01|C_2^*|$: we will show that $C^*_1$ and $C^*_2$ can be viewed as one cluster with $x$ as its center without harming the cost by much.
\end{enumerate}
\begin{eqnarray*}
\sum_{y \in C^*_1 \cup C^*_2} \norm{y-x}^2 & = & \sum_{y_1\in C^*_1} \norm{y_1-x}^2 + \sum_{y_2\in C^*_2} \norm{y_2-x}^2\\
& = &  \sum_{y_1\in C^*_1} \norm{y_1-x}^2 + \sum_{y_2\in C^*_2} \norm{(y_2^* - x) + (y_2 - y_2^*)}^2 \\
& \leq &  \sum_{y_1\in C^*_1} \norm{y_1-x}^2 + 2 |C^*_2|\norm{y_2^*-x}^2 + 2 \sum_{y_2\in C^*_2} \norm{y_2 - y_2^*}^2 \\
& \leq &  \sum_{y_1\in C^*_1} \norm{y_1-x}^2 + 2 \cdot 100\sum_{y_1\in B}\norm{y_1-x}^2 + 2 \sum_{y_2\in C^*_2} \norm{y_2 - y_2^*}^2 \\
& \leq &  201\sum_{y_1\in C^*_1} \norm{y_1-x}^2 + 2 \sum_{y_2\in C^*_2} \norm{y_2 - y_2^*}^2 \\
& \leq & 20100\cdot  cost(opt_2), 
\end{eqnarray*}
where the first inequality follows from Claim~\ref{clm:norm_of_sum_vectors}, the second from the definition of $B$, and the third from the definition of $Good.$
\end{proof}

\subsection*{Proof of Theorem~\ref{thm:k_upper_random}}
\begin{proof}%[proof of Theorem~\ref{thm:k_upper_random}]
We prove that for any constant $k$ with probability at least $0.9$
\begin{enumerate}
\item Number of centers is bounded by $O(k\log \frac{n}{k})$
\item $cost(alg)\leq \Theta(1)\cdot cost(opt_k)$
\end{enumerate}
These claims are proved in Claims~\ref{clm:k_2_upper_unknown_n_bound_centers} and \ref{clm:k_2_upper_unknown_n_bound_cost}. We prove the claim under the general cost function, see Section~\ref{sec:general_cost_function}.
\end{proof}

\begin{claim}\label{clm:k_2_upper_unknown_n_bound_centers}
 Algorithm~\ref{alg:k} takes as center at most $O(k\log \frac{n}{k})$, with probability at least $0.99$.
\end{claim}
\begin{proof}
To bound the number of centers the algorithm uses, we use a similar argument as in Theorem~\ref{thm:k_2_upper_random}. 
The algorithm takes the first $k$ points. 
For $i>k$ we want to find the probability that the $i$-th point is selected as a center.
This happens if among the $i$ points read so far, this point is one of the $k$ members in the Farthest-first-traversal.
The probability for this is $$\frac{k}{i}.$$
Thus the expected number of centers taken by the algorithm throughout its entire run is 
$$k + \sum_{j=k+1}^n \frac{k}{j} = k + k\left(\sum_{i=1}^n\frac1i - \sum_{i=1}^k\frac1i\right)=\Theta\left(k\log\frac{n}{k}\right).$$
%k + k\log\frac{n}{k}.$$
Use Markov's inequality to prove that with probability $0.99$ number of centers takes by the algorithm is at most $O\left(k\log\frac{n}{k}\right)$.% $$100k +100 k\log\frac{n}{k}=\Theta\left(k\log\frac{n}{k}\right),$$
%The last equality holds since w.l.o.g $k\leq\frac{n}{2}$ (otherwise the algorithm can simply take all the points as centers and the number of center is optimal, $\Theta(k)$) which implies that $\log\frac{n}{k}\geq 1$.
\end{proof}

\begin{claim}\label{clm:k_2_upper_unknown_n_bound_cost}
 The cost of Algorithm~\ref{alg:k} is bounded by $\Theta(1)\cdot cost(opt_k),$ with probability at least $0.96.$
\end{claim}
\begin{proof}
We will show that Algorithm~\ref{alg:k} is a $\Theta(1)$-approximation. To achieve that we will prove something stronger: when running the algorithm with parameter $k$, then for any $1\leq k'\leq k$  
with probability at least $1-\frac{4k'}{100k}$, $$cost(alg)\leq (10^5D^3)^{k'}k^{2k'}\cdot  cost(opt_{k'}).$$
Fix $k.$ We prove this claim by induction on $k'.$ 
For $k'=1$ the claim follows immediately, similarly to Algorithm~\ref{alg:k_1_n_unknown}, as we always take the first point as center. 

Denote the $k'$ optimal clusters by $C_1^*,\ldots, C_{k'}^*$. 
Focus on a cluster $C^*_i.$
The idea of the proof is that either the first point in $C^*_i$, most probably, is chosen as a center or the entire cluster can be added to another cluster. 
When the first point $x\in C^*_i$ arrives, its closest point is $y\notin C^*_i$. 
If $x$ is not chosen as a center, then there are $k$ centers from $k-1$ clusters that the distance between any two is larger than $\dis{x}{y}.$ There are two cases, as in Theorem~\ref{thm:k_2_upper_random}, either this is a common scenario, and then $C^*_i$ can be added to another cluster, or it's rear case and this means that with high probability $x$ will be chosen as a center.  

We define the set of \emph{good} points for a cluster $r\in[k]$ as the set of points that taking them as a center will not increase the cost by much. More formally, 
$$Good_r = \left\{y_r \in C^*_r | \sum_{x_i \in C^*_r} \dis{x_i}{y_r} \leq 100Dk \sum_{x_i \in C^*_r} \dis{x_i}{\mu^*(r)}\right\},$$
where  $\mu^*(r)$ is the optimal center for $C^*_r$.
From Lemma~\ref{lemma:random_point_in_cluster} and Markov's inequality we know that $Good_r$ is large. 
Specifically, $$\frac{|Good_r|}{|C^*_r|}\geq1-\frac{1}{50k}.$$
Thus, using union bound, we know that with probability at least $1-\frac{1}{50}$ the first point the algorithm encounters from each cluster $r$ is in in $Good_r$. 

Fix a cluster $C^*_{i}$ for some $i\in[k].$ 
The distance between $Good_i$ and a point $y$ is defined as 
$$dis(Good_i,y)=\min_{y'\in Good_i}\dis{y'}{y}.$$
Denote by $N$ the set of points that are not in the cluster $C^*_i$ (i.e., in $C^*-C^*_i$) and are one of the $\nicefrac{|C^*_i|}{100k}$ closest points to $Good_i.$
Denote the max distance in $N$ by $dis$, i.e., the distance that is $\nicefrac{|C^*_i|}{100k}$-closest to $Good_i$: 
$$dis = \max_{y'\in N}\min_{y\in Good_i} \dis{y'}{y}.$$
Denote the point the achieves this minimum in $Good_i$ by $x_i$.
%We denote the closest pair of points one from $Good_i$ and another from $Good_r$, where $r\neq i$ by $dis$
%$$dis = \min_{\substack{y\in Good_i\\ y'\in \cup_{r\neq i}Good_r}}\norm{y-y'}.$$
%Denote the pair of points that achieves this minimum by $(x_i,x_j)\in Good_i\times Good_j.$.
We want to define \emph{bad} points that can cause the algorithm not to take the first point from $C^*_i$ as a center. 
For that we first take for each cluster $r\neq i$ an arbitrary point in $x_r\in Good_r$.
%%We denote the closest point in $Good_i$ to some other $x_r$ by $dis$:
%%$$dis = \min_{y\in Good_i, r\neq i} \norm{y-x_r}.$$ 
Now we are ready to define the bad points: %Denote all the distance inside clusters other than $i$ that are larger than $dis$ by $B:$
$$B=\cup_{r\neq i}\left\{y\in C^*_r : \dis{y}{x_r} \geq \frac{dis}{2D}\right\}.$$ 

\begin{figure}
    \centering
    \includegraphics[scale=0.8]{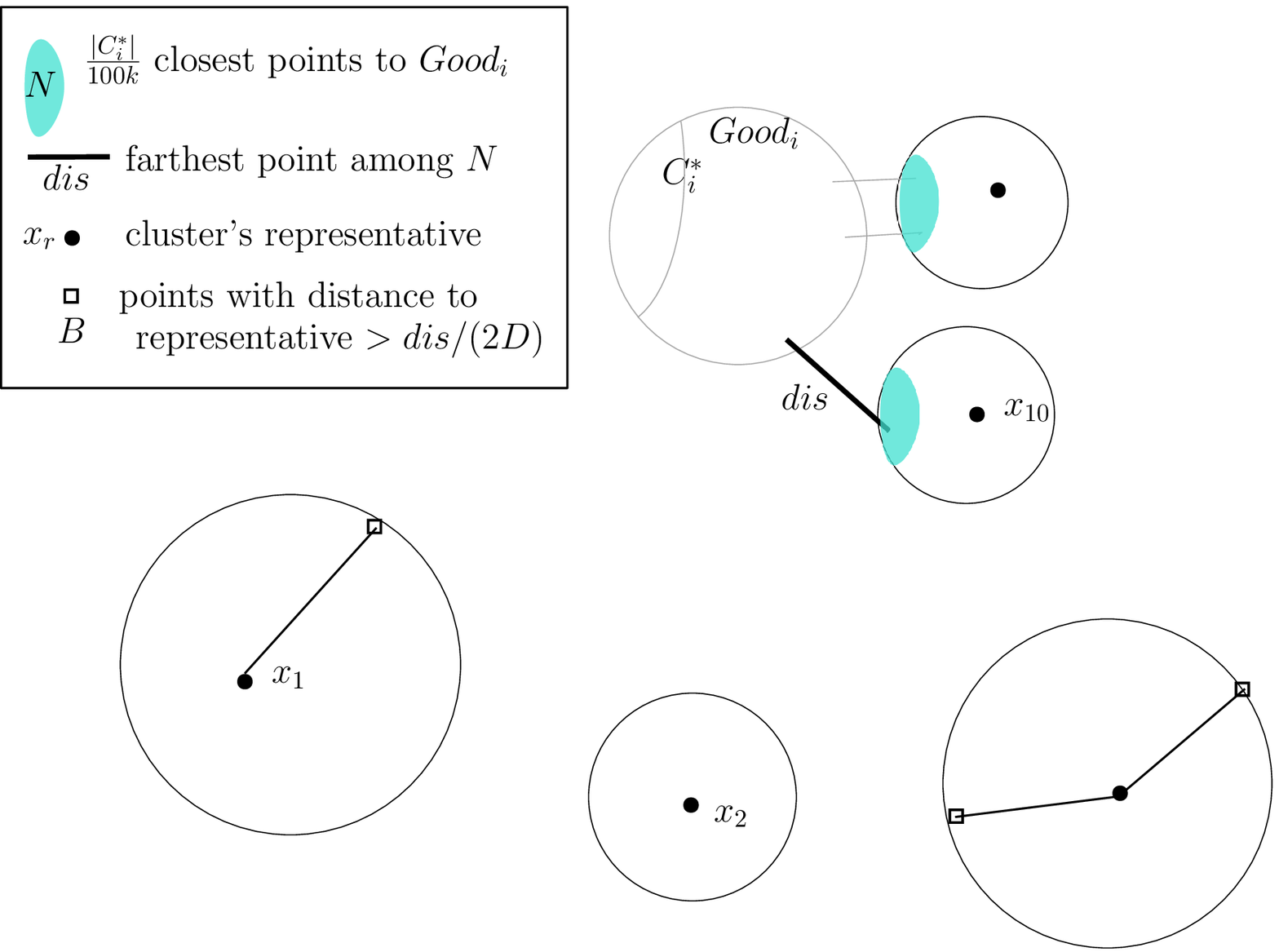}
    \caption{Notations used in the proof of Theorem~\ref{thm:k_upper_random}}
    \label{fig:my_label}
\end{figure}

There are two cases: either $B$ is large compared to $C^*_i,$ or it is not. 

% \begin{figure}
% \centering
% \begin{subfigure}{.5\textwidth}
%   \centering
%   \includegraphics[width=.99\linewidth]{k_upper_unknown_n_big_B}
%   \caption{$|B|\geq \frac{1}{100k}|C^*_i|$}
%   %\label{fig:sub1}
% \end{subfigure}%
% \begin{subfigure}{.5\textwidth}
%   \centering
%   \includegraphics[width=.99\linewidth]{k_upper_unknown_n_small_B}
%   \caption{$|B| <  \frac{1}{100k}|C^*_i|$}
%   %\label{fig:sub2}
% \end{subfigure}

\begin{figure*}[t!]
    \centering
    \begin{minipage}[t]{.5\textwidth}
        \centering
     %\subcaption{$|B|\geq \frac{1}{100k}|C^*_i|$}
     \includegraphics[width=.9\textwidth]{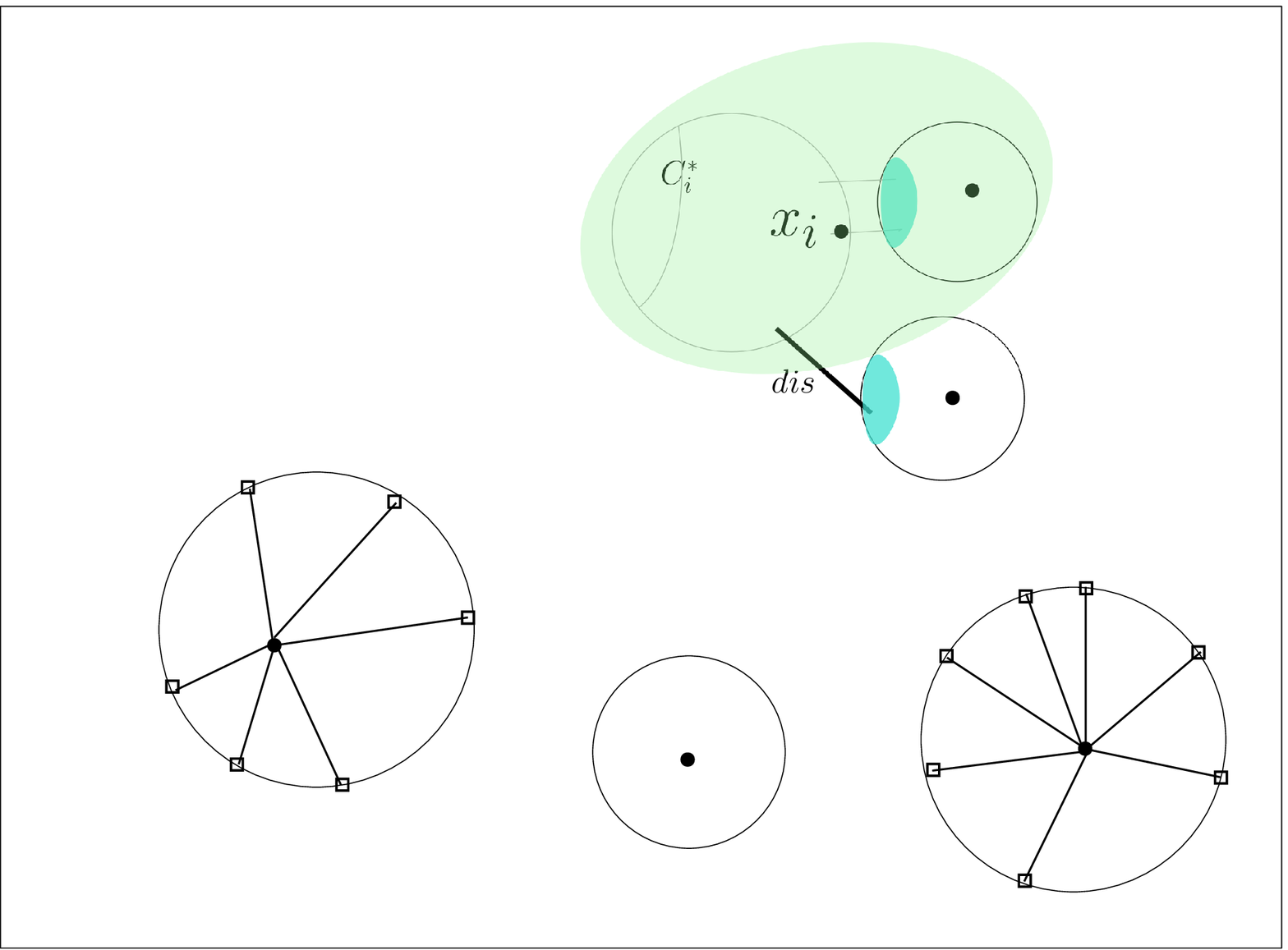}
     \end{minipage}%
     \begin{minipage}[t]{.5\textwidth}
         \centering
    %\subcaption{$|B| < \frac{1}{100k}|C^*_i|$}
     \includegraphics[width=.9\textwidth]{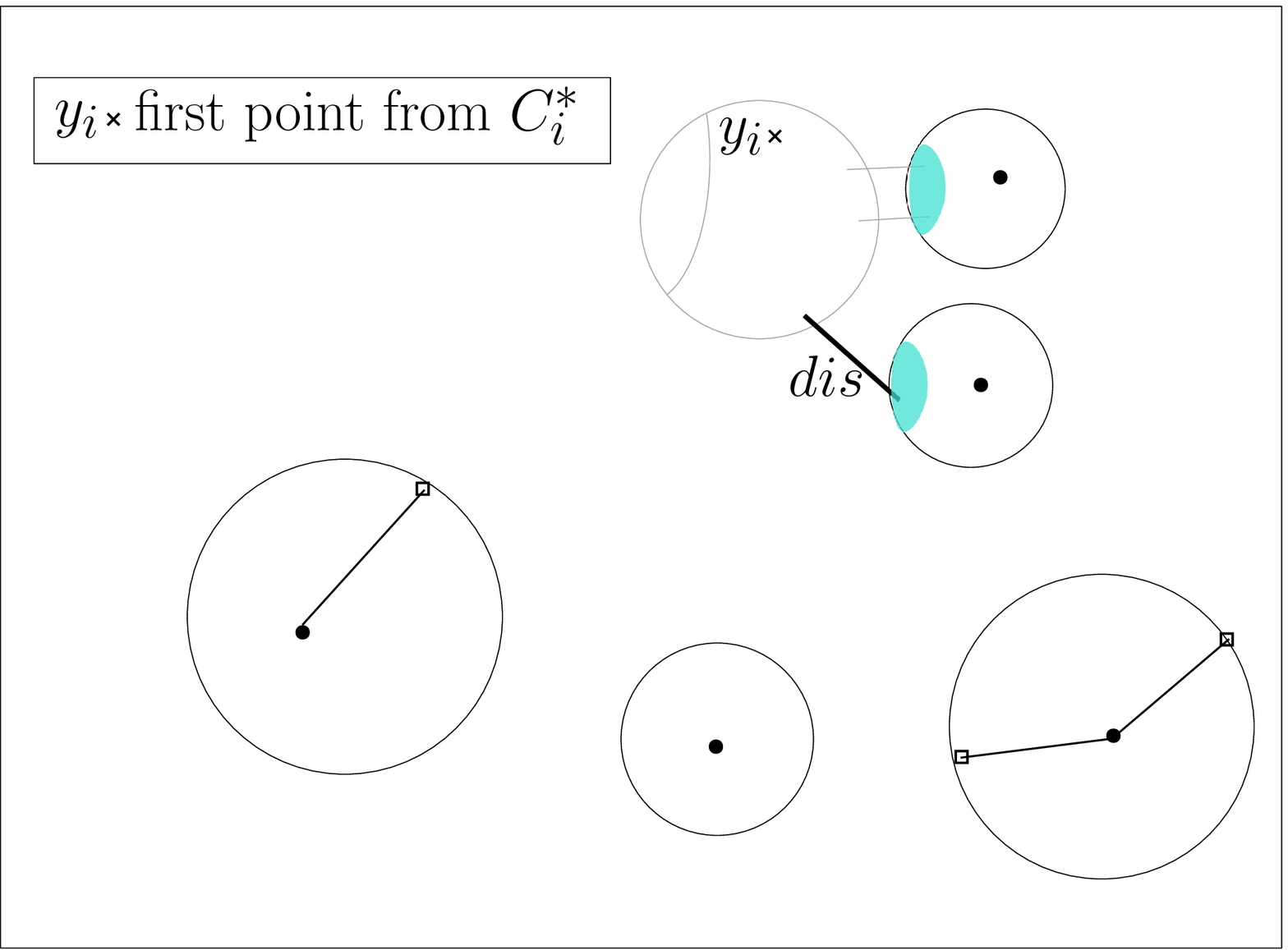}
     \end{minipage}
\caption{Bounding the cost for a cluster $C^*_i$. (a) if many points are in $B$ merge $C^*_i$ into its closest cluster (b) otherwise the first point in $C^*_i$ will be taken}
\end{figure*}

%Define the distance between a point $x$ and a set of points $C$ as the distance of $x$ to the closest point in $C$:
%$$dis(x,C) = \min_{y\in C}\norm{x-y}.$$
%Focus on the set of $0.01|C^*_{i_1}|$ points in $C^*-C^*_{i_1}$ that are closest to $C^*_{i_1},$ where the distance between a point $x$ and a set of points $C$ is equal to $$dis(x,C) = \min_{y\in C}\norm{x-y}.$$
%Denote by $max\_dis$ the largest distance in this set.
%Denote by $B$ the set of points that can cause the algorithm not to take the first point in $Good_i$ as a center:
%$$B=\cup_{r\in [k]} \{ y \in C^*_r : dis(y, C^*_r) \geq max\_dis\}$$  

%Take arbitrary member $x_r$ in each $Good_r$, where $r\in[k]-\{i_1\}.$
%Denote by $x(y)$ the function that returns for each  point $y\in C^*$ its center $x_1,\ldots x_k$. 
%Denote by $dis$ the shortest distance from $Good_{i_1}$ to some $y_r\in Good_r$:
%$$dis = \min_{r\neq i_1, y_t\in Good_{i_1}}\norm{y_r-y_t}$$

%Denote by $B$ all the points that their distance from their $ x_r$ is larger than $max\_dis$, i.e.,
%$$B=\cup_{r\in [k]} \{ y \in C^*_r : \norm{y-x_r} \geq dis\}$$  
%There are two cases

\underline{If $|B|\geq \frac{1}{100k}|C^*_i|$:} we will show that taking the $k-1$ centers --- all the centers $x_r$'s without $x_{i}$ --- is a good enough clustering and then the claim follows by the induction assumption. 
To prove that, first take, among all points in $N$, one $y_j\in N$ that it's closest to its representative $x_j$:
$$y_j = \argmin_{\substack{j\in[k]-\{i\}\\ y\in N\cap C^*_j}} \dis{y}{x_j}.$$
In particular, since $y_j\in N$ and the definitions of $x_i$ and $dis$ we get that $$\dis{x_i}{y_j}\leq dis.$$
By the definition of $y_j$ we also have that for every $r\in[k]-\{i\}$ and $y_r\in B\cap C_r^*$, $$\dis{y_j}{x_j}\leq\dis{y_r}{x_r}.$$
We can use these two observations to bound $cost(opt_{k'-1})$ in terms of $cost(opt_{k'})$: 
\begin{eqnarray*}
cost(opt_{k'-1}) &\leq& \sum_{y\in C^*_i} \dis{y}{x_j} + \sum_{r\neq i}\sum_{y\in C^*_r} \dis{y}{x_r}\\
&\leq& D\sum_{y\in C^*_i} \dis{y}{x_{i}} + D^2\sum_{y\in C^*_i}\dis{x_{i}}{y_j} + D^2\sum_{y\in C^*_i}\dis{y_j}{x_j} + \sum_{r\neq i}\sum_{y\in C^*_r} \dis{y}{x_r}\\
&\leq&  D^2\sum_{y\in C^*_i}\dis{x_{i}}{y_j} + D^2\sum_{y\in C^*_i}\dis{y_j}{x_j} + D\sum_{r}\sum_{y\in C^*_r} \dis{y}{x_r}\\
&=& D^2|C^*_i|\dis{x_{i}}{y_j} + D^2|C^*_i|\dis{y_j}{x_j} + D\sum_{r}\sum_{y\in C^*_r} \dis{y}{x_r}\\
&\leq& D^2|C^*_i|dis + D^2|C^*_i|\dis{y_j}{x_j} + D\sum_{r}\sum_{y\in C^*_r} \dis{y}{x_r}\\
&\leq& 100D^2k|B|dis + 100D^2k\frac{|C^*_i|}{100k}\dis{y_j}{x_j} + D\sum_{r}\sum_{y\in C^*_r} \dis{y}{x_r}\\
%&=& 800|B|\frac{dis^2}{4} + 2\sum_{r}\sum_{y\in C^*_r} \norm{y-x_r}^2\\
&\leq& 201D^2k\sum_{r}\sum_{y\in C^*_r} \dis{y}{x_r}\leq 10^5D^3k^2\cdot  cost(opt_{k'})\\
%&& \sum_{y_1\in C_{i_1}} \norm{y_1-x(y_1)}^2  + 100 k \sum_{r\neq i_1} \norm{y_r-\mu_r}^2\\
%&\leq&100k cost(opt)
\end{eqnarray*}
where the second inequality follows from Claim~\ref{clm:norm_of_sum_vectors} (or Inequality~\ref{eq:generalized_triangle_inequality} for the general cost), the forth since $y_j\in N$, the sixth because $y_j$ is closest to its representative from all $N$, and the last inequality follows from the definition of good. 

From the induction assumption we know that 
\begin{eqnarray*}
cost(alg) &\leq& (10^5D^3)^{k'-1}(k-1)^{2k'}\cdot cost(opt_{k'-1})\\
& \leq& (10^5D^3)^{k'-1}k^{2k'}\cdot (10^5D^3)^2\cdot  cost(opt_{k'})\\
&=&(10^5D^3)^{k'}k^{2k'}\cdot  cost(opt_{k'})
\end{eqnarray*}
And this proves the claim in this case. Let's move on to the next case. 
%|C^*_{i_1}|\cdot dis
%\leq 101 \cdot cost(opt)

\underline{If $|B| < \frac{1}{100k}|C^*_i|$:}
We will show that with probability at least $1-\frac{4}{100k}$, the first point from $C^*_{i}$ will be chosen as a center and is in $Good_{i}$. Then, if there is no cluster $C^*_i$ with $|B|\geq |C^*_i|/(100k)$ we can use union bound over all the $k'$ centers will finish the proof. 
Focus on the time where the first point in $C^*_i,$ $y_i \in C^*_i$ was given.
With probability at least $1-\nicefrac{1}{50k}$ it is in $Good_i$.
With probability at least $1-\frac{1}{100k}$ a point from $|B|$ was not chosen yet (as $B$ is much smaller than $C^*_i$).
With probability at least $1-\frac{1}{100k}$ a point from $N$ was not chosen yet (as $N$ is much smaller than $C^*_i$), thus minimal distance from a point in $Good_i$ to another is at least $dis.$
%From the definition of good, we have that $$|\cup_{r\neq i}Good_r|\geq \left(1-\frac1{50k}\right)|\cup_{r\neq i}C^*_r|.$$ 

For the sake of contradiction, let us assume that $y_i$, the first point from $C^*_i$, was not chosen as a center. 
Then all points $S$ returned by the farthest-first-traversal algorithm must came from at most $k-1$ optimal clusters.
From the Pigeonhole principle, there are two points $y_1,y_2\in S$ that are in the same cluster $C^*_j$. 
%This implies that with probability at least $1-\nicefrac{1}{50k}$ it holds that  $$dis(S,y_i)=\min_{y\in S}\norm{y-y_i}\geq dis.$$
From Lemma~\ref{lemma:farthest_first_traversal} we have that $$dis\leq \dis{y_1}{y_2}.$$
From Claim~\ref{clm:norm_of_sum_vectors} (or Inequality~\ref{eq:generalized_triangle_inequality}) we have that $\dis{y_1}{y_2}\leq D\cdot \dis{y_1}{x_j} + D\cdot\dis{x_j}{y_2}.$
This implies that $\dis{y_1}{x_j}\geq dis/(2D)$ (or $\dis{y_2}{x_j}\geq dis/(2D)$).
This means that $y_1\in B$. 
Which is a contradiction to the assumption that no point in $B$ was chosen yet. 
%The probability that B will appear before $Good_i$ is small

%Take all the points $D'$ in $Good_{i_1}$ that does not have a member in $D$ (define better). 
%Note a few things (i) $|D'|\geq 0.02|C^*_{i_1}|.$ (ii) The distance of theses points in $D'$ is bigger than all the inner (iii) the set of 'bad' points is the verices in $B$ which is smaller than $B.$ 

\end{proof}

\paragraph{Comparison between Algorithm~\ref{alg:k} and \cite{liberty16}:}
\cite{liberty16} design an online clustering  adaptation to the $k$-means++ algorithm \cite{arthur07}. 
%Inherently, their algorithm cannot get the optimal bound, as will be explained next. %, in contrast to our algorithm that is able to achieve optimal results. 
%In this paper we use a different algorithm that enables to drastically improve \cite{liberty16}. 
In the $k$-means++ algorithm a point is chosen with probability $\nicefrac{\text{distance}}{Z}$, where $Z$ is the normalization factor.
In the online case $Z$ is unknown as it depends on future points and thus cannot be calculated. 
Instead, \cite{liberty16} suggests to start with some small $Z$ and to keep increasing it by a factor of $2$.
This implies that the number of centers taken by the algorithm depends on the scale of the data.
%As an illustration, if the distances are about $s$, then $\log s$ centers are needed till the ``probability'' to take a center will be less than $1.$ 
The scale of data $D$ is summarized in the  aspect ratio parameter $\gamma=\frac{\max_{v,v'\in D}\norm{v-v'}}{\min_{v,v'\in D}\norm{v-v'}}.$
 \cite{liberty16} takes $O(\log n (\log\gamma +\log n))$ centers and achieves $O(\log n)$-approximation. 
Inherently, this algorithm depends on $\gamma$ as it must reach the scale of the data to be able to take only a small subset of points as centers. Because of the dependence on $\gamma$, their algorithm cannot achieve the optimal bound.
In this paper, we improve both the quality of the approximation and the number of centers by the algorithm to the optimal values in case the order is random (see Algorithm~\ref{alg:k}). %and the number of centers to a constant, 

\subsection{General cost function}
\begin{claim}\label{thm:general_k_1_worst_n_unknown_lower}
Assume there are $n$ points with $\dis{x_i}{x_j}=|j-i|.$ For any clustering algorithm that is not given the dataset size in advance and is a $c$-approximation compared to $opt_1$, there are $n$ data points and an ordering of them such that the algorithm must take $\Omega(\log_c(n))$ centers with probability at least $0.8$.  
\end{claim}
\begin{proof}
The proof is the same as in Theorem~\ref{thm:k_1_worst_n_unknown}, where now the examples are $x_1,\ldots,x_n$ instead of $1,\ldots,n.$ 
\end{proof}

\begin{claim}\label{thm:general_k_1_worst_n_unknown_upper}
For any $c>1$ there is an algorithm that obtains $O(cD)$-approximation with $O(\log_c n)$ centers, no matter what the order is and even if $n$ is unknown.
\end{claim}
\begin{proof}
Similarly to the proof of Theorem~\ref{thm:k_1_worst_n_unknown_upper}, in the one before the last iteration, $n'$ is big compared to $n$, it's at least a fraction $1/2c$ of $n$. From Lemma~\ref{lemma:general_cost_expected_cost} and Markov's inequality we get that for at most $\frac{n}{20\cdot c}$ of the data points $x$ it holds that $cost(x) >40\cdot D \cdot c\cdot cost(opt_1).$ Thus with probability at least $0.9$ the algorithm chooses a data point $x$ out of the $n'$ points at the one before the last iteration with $cost(x)\leq 400cD\cdot cost(opt_1).$
\end{proof}

\section{Random order}
In this section we present and prove general claims in case data appears in random order. These claims will be helpful to analyze Algorithm~\ref{alg:k_2_upper_known_n}. In Section~\ref{apx:random_choice_general_claims} we analyze the expected appearance of a predetermined set in a random sample. In Section~\ref{apx:random_clustering_general_claims} we connect the optimal clustering and a clustering based on a random sample. 

\subsection{Random sample}\label{apx:random_choice_general_claims}
The next claim shows that for any predetermined set, with a high probability the sample is a good representation of this set. This probability depends on the set and sample sizes; if they are bigger the probability is higher.  
\begin{claim}\label{clm:random_from_each_part_tight_bound}
Fix a dataset $D$, a subset $S\subseteq D$ and a scalar $0\leq\beta\leq1$. 
Assume a subset $S'\subseteq D$ with $\frac{|S'|}{|D|}=\beta$ is chosen uniformly at random. 
Then, for any $a>0,$
$$\Pr\left(\left| |S \cap S'|-|S|\frac{|S'|}{|D|}\right| \geq  \sqrt{\frac{|S'||S|}{a|D|}}\right) \leq a $$ 
\end{claim}
\begin{proof}
Denote by $X$ the random variable that is equal to $|S \cap S'|$. 
Order the members in $S$ in some arbitrary order. 
Denote by $X_i$, $i=1,\ldots,|S|$, the binary variable that is equal to $1$ if the $i$'th member in $S$ is also in $S'$, and $X_i=0$ otherwise.
%Denote $\beta := \nicefrac{|S'|}{|D|}.$
Note two basic properties of the random variables $X_i$'s:
\begin{eqnarray}\label{eq:clm_random_from_each_part_expectation}
\E[X_i] = \beta,
\end{eqnarray}
and for every $i\neq j$ it holds that 
\begin{eqnarray}\label{eq:clm_random_from_each_part_covariance}
\E[X_i X_j] = \Pr(X_i = 1 \wedge X_j =1)=\frac{|S'|(|S'|-1) (|D|-2)!}{|D|!}=\frac{|S'|(|S'|-1)}{|D|(|D|-1)}\leq\beta^2,
\end{eqnarray}
where the second equality holds since we can view the process of taking $S'$ as if we order all the points in $D$ and then take the first $|S'|$ members; with this view, the $i$th member has $|S'|$ places and the $j$th member has $|S'|-1$ places and then arrange all the other members, $(|D|-2)!$ options. 

We now analyze the expectation and the variance of $X$.  $$\E[X] = \sum_{i=1}^{|S|}{\E[X_i]}=\beta |S|.$$
%To bound the probability from expectation we bound the variance 
Bounding the variance of $X$
\begin{eqnarray*}
Var[X] &=& \E[X^2] - \left(\E[X]\right)^2 \\
&=& \E\left[\left(\sum_{i=1}^{|S|} X_i\right)^2\right] - \left(\sum_{i=1}^{|S|} \E[X_i]\right)^2\\
&=& \E\left[\sum_{i \neq j} X_i X_j\right] + \E\left[\sum_{i=1}^{|S|} X_i^2\right] -  \sum_{i,j} \E[X_i]\E[X_j]\\
&=& \sum_{i \neq j} \E\left[X_i X_j\right] + \sum_{i=1}^{|S|}\E\left[ X_i\right] -  \sum_{i,j} \E[X_i]\E[X_j]\\
&\leq& \beta^2|S|^2 + \beta|S| - \beta^2|S|^2=\beta|S|
\end{eqnarray*}
where the fourth equality follows from the fact that $X_i^2=X_i$ as $X_i$ is a binary random variable and the inequality follows from Equation~\ref{eq:clm_random_from_each_part_expectation} and  Equation~\ref{eq:clm_random_from_each_part_covariance}. 
Next we use Chebyshev's inequality which is the following bound for any $C>0$: $$\Pr(|X-\E[X]| \geq C)\leq\frac{Var[X]}{C^2}.$$
In our case, take $$C=\sqrt{\frac{Var[X]}{a}}\leq\sqrt{\frac{\beta|S|}{a}}$$ and Chebyshev's inequality implies that $$\Pr\left(|X-\beta|S|| \geq \sqrt{\frac{\beta|S|}{a}}\right)\leq a.$$
%Equivalently, $$\Pr\left(\left|\frac{X}{|D|}-\frac{\beta|S|}{|D|}\right| \geq \frac{C}{|D|}\right)\leq\frac{\beta|S|}{C^2}.$$
%Take $C=\sqrt{\frac{\beta |S|}{a}}$ and get that $$\Pr\left(\left|\frac{X}{|D|}-\frac{\beta|S|}{|D|}\right| \geq \sqrt{\frac{\beta |S|}{a|D|^2}}\right)\leq a.$$
\end{proof}

The next claim shows that for any large predetermined set, the sample is a good representation of this set. More specifically, for any large enough fixed set $S$, its proportion, $\nicefrac{|S\cap S'|}{|S'|}$, in the sample $S'$, is roughly its proportion, $\nicefrac{|S|}{|D|}$, in the dataset $D$. Equivalently, the proportion of points, $\nicefrac{|S\cap S'|}{|S|}$, in $S$ that are in the sample, is about $\beta |S|$ where $\beta=\nicefrac{|S'|}{|D|}$ is the fraction of sample of the entire dataset.    
\begin{claim}\label{clm:random_from_each_part}
Fix a dataset $D$, a subset $S\subseteq D$ and a scalar $0\leq\beta\leq1$. 
Assume a subset $S'\subseteq D$ of size $|S'|=\beta|D|$ is chosen uniformly at random. 
For any $a>0,$ if $|S|\geq \frac{4}{\beta a}$ then  
$$\Pr\left(\frac{\beta}2|S|\leq |S \cap S'| \leq 2\beta|S|\right)\geq 1-a$$
%$$\Pr\left(\left| \frac{|S \cap S'|}{|D|}-\frac{|S|}{|D|}\frac{|S'|}{|D|}\right| \geq  \sqrt{\frac{|S'||S|}{a|D|^3}}\right) \leq a $$ 
\end{claim}
\begin{proof}
From Claim~\ref{clm:random_from_each_part_tight_bound} we know that 
$$\Pr\left(\left| |S \cap S'|-\beta|S|\right| \geq  \sqrt{\frac{\beta}{a}|S|}\right) \leq a.$$
To prove the claim, it is enough to prove that  $$\sqrt{\frac{\beta}{a}|S|}\leq \frac\beta2|S|\Leftrightarrow\frac{2}{\sqrt{\beta a}}\leq\sqrt{|S|}\Leftrightarrow\frac{4}{\beta a} \leq |S|$$
\end{proof}

% \begin{fact}\label{fact:exp_and_linear}
% For any scalar $x\in(0,0.5)$
% $$1-x\leq e^{-x}\leq1-\frac{x}{2}$$
% \end{fact}

In the next claim we prove that is a set is small, most likely it won't appear in a sample. 
\begin{claim}\label{clm:small_random_set}
Fix a dataset $D$, a subset $S\subseteq D$ and a scalar $0\leq\beta\leq1$. 
Assume a subset $S'\subseteq D$ of size $|S'|=\beta|D|$ is chosen uniformly at random. 
For any $a>0,$ if $|S|\leq \frac{a}{\beta}$ the probability that a member from $S$ will be chosen is at most $a$. 
\end{claim}
\begin{proof}
The probability that a specific point be taken is $\beta.$ Using union bound, the probability that at least one point from $S$ is taken is bounded by $\beta|S|\leq a.$
% The probability that a member from $S$ will be taken is equal to $1-(1-\beta)^{|S|}$ we want to show that it is at most $a.$ Or in different words that $$1-a\leq (1-\beta)^{|S|}.$$
% From Fact~\ref{fact:exp_and_linear}, it is enough to prove that $e^{-a}\leq e^{-2\beta|S|}\Leftrightarrow |S|\leq\frac{a}{2\beta}.$
\end{proof}

\begin{claim}\label{clm:random_sample_at_least_in_set}
 Suppose $r$ random points are chosen uniformly at random from $C = A \dot{\cup} B.$ The probability that \emph{all} $r$ points are in $B$ is bounded by $\left(\frac{|B|}{|C|}\right)^r.$
\end{claim}
\begin{proof}
Denote $n_A=|A|, n_B=|B|, n=|C|=n_A+n_B.$
The random sample can be viewed as if $C$ is ordered uniformly at random and the first $r$ points are the sample. The probability that in a random order the first $r$ elements are \emph{all} from $B$ is 
\begin{eqnarray}\label{eq:random_sample_all_from_B}
\frac{(n-r)(n-r-1)\cdot\ldots\cdot(n-r-n_A+1)n_B!}{n!},
\end{eqnarray}
the first element in $A$ has $n-r$ locations out of $n$ (all but the first $r$ elements) the last element in $A$ has $n-r-n_A+1$ locations, and the $n_B$ members in $B$ has no constraints in ordering them. 

For any $y\geq x>a\geq0$ it holds that  $(x-a)/(y-a)\leq x/y$, thus Expression~\ref{eq:random_sample_all_from_B} is equal to $$\frac{n_B}{n}\cdot\frac{n_B-1}{n-1}\cdot\ldots\cdot\frac{n_B-r+1}{n-r+1}\leq\left(\frac{n_B}{n}\right)^r.$$ 
\end{proof}

\begin{corollary}\label{cor:random_sample_enough_points_hit_set}
Fix $\delta\in(0,1)$ and $C = A \dot{\cup} B.$ 
Suppose $\left\lceil{\frac{|C|}{|A|}\log\frac1{\delta}}\right\rceil$ random points are chosen uniformly at random from $C$. The probability that at least one element out of the random sample is in $A$ is at least $1-\delta.$
\end{corollary}
\begin{proof}
Follows from Claim~\ref{clm:random_sample_at_least_in_set} and
$$\left(1-\frac{|A|}{|C|}\right)^{\frac{|C|}{|A|}\log\frac1{\delta}}\leq\delta.$$
\end{proof}

\subsection{Random clustering}\label{apx:random_clustering_general_claims}
In this section we explore the connection between optimal clustering and one that is based on a sample of the data. Specifically, we show that clustering that is based on a random sampling is also good clustering for large optimal clusters. We focus on the case that there is a clustering with $k'$ optimal clusters $C^*=(C^*_1,\ldots, C^*_{k'})$. We get a random sample $M$ of size $\alpha n$ and construct an optimal clustering $C^{M}=(C^{M}_1,\ldots,C^{M}_k)$ for those $\alpha n$ points with $k\geq k'$ clusters. In fact, $C^{M}$ does not have to be an optimal clustering for the $\alpha n$ points, merely a $\Theta(1)$-approximation.  

To make our claims more general, we consider an arbitrary cost that its distance function satisfy a version of a triangle inequality: there is a \emph{constant} $D$ such that  
\begin{equation}%\label{eq:generalized_triangle_inequality}
\forall u,v,w.\quad d(u,v)\leq D\cdot(d(u,w)+d(w,v)).
\end{equation}
We call such a cost a \emph{$D$-cost}.

We first prove that if an optimal clustering $C^*_i$ is large enough, then there is a center in $C^{M_1}$ which is a good enough center for the \emph{entire} cluster $C^*_i$. 
\begin{claim}\label{clm:known_bound_cost_large_cluster_exists_good_cluster}
Fix an $a$-approximation algorithm with $k\geq k'$ clusters for a random sample of size $\alpha n$, $D$-cost, and $\delta\in(0,1)$. For any optimal cluster $C^*_i$, $i\in[k']$, $|C^*_i|\geq \frac{4}{\alpha\delta}$, with probability at least $1-\delta$ there is a center $c^{M}_{i'}$ in $C^{M}$ such that $$\sum_{x\in C^*_i}\dis{x}{c^{M}_{i'}}\leq\frac{5kD^2a}{\alpha}\cdot cost(opt_{k'}).$$
\end{claim}
\begin{proof}
Denote by $M$ the random sample of size $|M|=\alpha n.$
From Claim~\ref{clm:random_from_each_part}, with probability at least $1-\delta$ there are at least $\frac{\alpha|C^*|}{2}$ points in $C^*_i\cap M$. 
There is a cluster $C^M_{i'}\in C^M$ such that there are at least $\frac{\alpha|C^*|}{2 k}$ from $C^*_i\cap C^M_{i'}.$  By Inequality~\ref{eq:generalized_triangle_inequality} (or Claim~\ref{clm:norm_of_sum_vectors} for the $k$-means cost) with $\mu$ the optimal center of $C^*_i$ it holds that 
\begin{eqnarray*}
\sum_{x\in C^*_i}d(x,c^M_{i'}) &\leq& D\sum_{x\in C^*_i}d(x,\mu)+D\sum_{x\in C^*_i}d(x,c^M_{i'})\\
&\leq& D\cdot cost(opt_{k'}) + D|C^*_{i}|d(\mu,c^*_i)
\end{eqnarray*}
Let's focus on the second term. We again use  Inequality~\ref{eq:generalized_triangle_inequality} and get that 
\begin{eqnarray*}
D|C^*_{i}|d(\mu,c^*_i) &=& \frac{2D k}{\alpha}\cdot \frac{|C^*_i|}{2\alpha k}d(\mu,c^M_{i'})\\
&\leq& \frac{2D^2 k}{\alpha} \sum_{x\in C^*_i\cap C^M_{i'}} d(x,\mu) + \frac{2D^2 k}{\alpha} \sum_{x\in C^*_i\cap C^M_{i'}} d(x,c^{M}_{i'})\\
&\leq& \frac{2 k D^2}{\alpha} cost(opt_{k'}) + \frac{2 k D^2}{\alpha}\cdot a \cdot  cost(opt_{k})\\
&\leq& \frac{4 k D^2 a}{\alpha} cost(opt_{k'}), 
\end{eqnarray*}
where in the second inequality we use the fact that $C^M$ is an a-approximation and in the last inequlity we use the fact that $k\geq k'$ and thus $cost(opt_k)\leq cost(opt_{k'}).$ 

\end{proof}

In the second auxiliary claim we prove is that for any large optimal cluster $C^*_i$ and for any cluster in $C^{M}$ with center $c^{M}$ that contains $A\subseteq C^*_i$ members, its center $c^{M}$ is a good center for the points in $A.$
\begin{claim}\label{clm:known_bound_cost_large_cluster_contains_many_points_good_cluster}
Fix an $a$-approximation algorithm with $k\geq k'$ clusters for a random sample of size $\alpha n$, $D$-cost, and $\delta\in(0,1)$. 
For any optimal cluster $C^*_i$, $i\in[k']$, $|C^*_i|\geq \frac{4}{\alpha\delta}$, with probability at least $1-\delta$ for any cluster $C^{M}_i$ with center $c^{M}_i$ that contains points $A\subseteq C^*_i\cap C^{M}_i$ it holds that $$\sum_{x\in A}\dis{x}{c_i^{M}}\leq \frac{5kD^2a}{\alpha}\cdot cost(opt_{k'}).$$
\end{claim}
\begin{proof}
By Claim~\ref{clm:known_bound_cost_large_cluster_exists_good_cluster} we know that with probability at least $1-\delta$ there is a center $c^{M}_{i'}$ in $C^{M}$ such that $$\sum_{x\in C^*_i}\dis{x}{c^{M}_{i'}}\leq\frac{5kD^2a}{\alpha}\cdot cost(opt_{k'}).$$ So for any cluster $C^{M}_i\in C^{M}$ we can deduce that 
\begin{eqnarray*}
\sum_{x\in A}\dis{x}{c_i^{M}} \leq \sum_{x\in A}\dis{x}{c_{i'}^{M}}\leq \sum_{x\in C^*_i}\dis{x}{c_{i'}^{M}}\leq\frac{5kD^2a}{\alpha}\cdot cost(opt_{k'}).
\end{eqnarray*}
where the first inequality follows from the fact that points $x\in A$ are closer to $c^{M}_i$ than $c^{M}_{i'}$ and the second inequality holds because $A\subseteq C^*_i.$
\end{proof}

In the third auxiliary claim we prove that the last claim implies that $c^{M_1}$ is close to $c^*_i.$

\begin{claim}\label{clm:known_n_bound_cost_large_cluster_close_to_center}
Fix an $a$-approximation algorithm with $k\geq k'$ clusters for a random sample of size $\alpha n$, $D$-cost, and $\delta\in(0,1)$. 
For any optimal cluster $C^*_i$, $i\in[k']$, $|C^*_i|\geq \frac{4}{\alpha\delta}$, with probability at least $1-\delta$ for any cluster $C^{M}_i$ with center $c^{M}_i$ that contains points $A\subseteq C^*_i\cap C^{M}_i$ it holds that $$|A|\cdot\dis{c^{M}_i}{c^*}\leq  \frac{6kD^3a}{\alpha}\cdot cost(opt_{k'}).$$
\end{claim}
\begin{proof}
%Denote by $A$ the set of points that contains at least $m$ points. 
Use Claim~\ref{clm:known_bound_cost_large_cluster_contains_many_points_good_cluster}
\begin{eqnarray*}
|A|\cdot \dis{c^*_i}{c^{M}_i} &\leq& D\sum_{x\in A}\dis{c^*_i}{x}+\dis{x}{c^{M}_i}\\
&\leq& D\cdot cost(opt_{k'}) + \frac{5kD^3a}{\alpha}\cdot cost(opt_{k'})
\end{eqnarray*}
\end{proof}

The fourth, and the last, auxiliary claim shows that if there is a cluster in $C^{M}$ that contains many points from two different optimal clusters, then these clusters can be merged, without harming the cost by much. 

\begin{claim}\label{clm:known_n_bound_cost_large_cluster_two_cluster_are_close}
Fix an $a$-approximation algorithm with $k\geq k'$ clusters for a random sample of size $\alpha n$, $D$-cost, and $\delta\in(0,1)$. 
For any optimal clusters $C^*_i$,$C^*_j$, $i,j\in[k']$ with $|C^*_i|,|C^*_j|\geq \frac{8}{\alpha\delta}$, with probability at least $1-\delta$ if there is a cluster in $C^{M}$ that contains at least $\zeta|C^*_i|$ points from $C^*_i$ and at least $\eta|C^*_i|$ from $C^*_j$, then with probability at least $1-\delta$ it holds that $$cost(opt_{k'-1})\leq\frac{13kD^5a}{\alpha\min(\zeta,\eta)}\cdot cost(opt_{k'})$$
\end{claim}
\begin{proof}
We want to bound the cost of the following clustering with ${k'}-1$ centers: $c^*_1,\ldots, c^*_{k'}$ without $c^*_i$ and all points in $C^*_i$ will be assigned to $c^*_j$. The cost is equal to 
\begin{eqnarray*}
\sum_{x\in C^*_i} \dis{x}{c^*_j} + \sum_{r\neq i} \sum_{x\in C^*_r} \dis{x}{c^*_r}
\end{eqnarray*}
Let us bound the first sum using Claim~\ref{clm:known_n_bound_cost_large_cluster_close_to_center}, with probability $1-\delta$ 
\begin{eqnarray*}
\sum_{x\in C^*_i}\dis{x}{c^*_j} &\leq& D^2\sum_{x\in C^*_i}\dis{x}{c^*_i}+D^2\sum_{x\in C^*_i}\dis{c^*_i}{c^{M}_i}+D\sum_{x\in C^*_i}\dis{c^{M}_i}{c^*_j}\\
&=& D^2\sum_{x\in C^*_i}\dis{x}{c^*_i}+\frac{D^2}{\zeta}\zeta| C^*_i|\dis{c^*_i}{c^{M}_i}+\frac{D}{\eta}\eta| C^*_i|\dis{c^{M}_i}{c^*_j}\\
&\leq& D^2cost(opt_{k'}) +  \frac{6kD^5a}{\alpha\zeta}\cdot cost(opt_{k'})
+ \frac{6kD^4a}{\alpha\eta}\cdot cost(opt_{k'})\\
&\leq& \frac{13kD^5a}{\alpha\min(\zeta,\eta)}\cdot cost(opt_{k'})
\end{eqnarray*}
%The first sum is bounded by $D^2\cdot cost(opt_{k'})$. The second and third terms are bounded by $\left(\frac{103D^5k}{b\alpha} + \frac{103D^5k}{a\alpha}\right)\cdot cost(opt_{k'}).$%$(\frac{8004k}{b\alpha} + \frac{8004k}{a\alpha})\cdot cost(opt_k).$
\end{proof}

\section{Auxiliary claims}\label{apx:auxiliary_claims}
%\subsection*{Proof of Lemma~\ref{lemma:random_point_in_cluster}}
\begin{proof}[%proof 
of Lemma~\ref{lemma:random_point_in_cluster}]
The proof consists of simply rewriting the two expressions. The right-hand side is equal to 
\begin{eqnarray*}
\E_{j\in[n]}\left[\sum_{i=1}^n\norm{x_i-x_j}^2\right] &=& \frac1n\sum_{j=1}^n\sum_{i=1}^n \norm{x_i-x_j}^2\\
 &=& \frac1n\sum_{j=1}^n\sum_{i=1}^n \left( \norm{x_i}^2+\norm{x_j}^2-2\inner{x_i}{x_j}\right)\\
   &=&2 \sum_{i=1}^n \norm{x_i}^2-\frac{2}n\sum_{i,j}\inner{x_i}{x_j}\\
\end{eqnarray*}
The second expression is equal to twice the following expression
\begin{eqnarray*}
\sum_{i=1}^n\norm{x_i-\mu}^2 &=& \sum_{i=1}^n\norm{x_i-\frac1{n}\sum_{j=1}^nx_j}^2\\
&=&\sum_{i=1}^n\norm{x_i}^2 -\frac{2}{n}\sum_{i=1}^n\inner{x_i}{\sum_{j=1}^nx_j}+\frac{1}{n}\norm{\sum_{j=1}^nx_j}^2\\
&=&\sum_{i=1}^n \norm{x_i}^2 - \frac1n\sum_{i,j}\inner{x_i}{x_j}
\end{eqnarray*}
\end{proof}

\begin{claim}\label{clm:geometric_series_squared_diff_bound}
For any scalar $q\geq 6$ and an integer $n\geq 1$ it holds that  $$\sum_{i=1}^n q^i (n-i)^2 \leq 6\cdot q^{n-1}.$$
\end{claim}
\begin{proof}
\begin{eqnarray*}
\sum_{i=1}^n q^i (n-i)^2 &=& q^n \sum_{i=1}^n \left(\frac1{q}\right) ^{n-i}(n-i)^2 \\
&=& q^n \sum_{j=0}^{n-1} \left(\frac1{q}\right) ^{j}j^2 \\
&\leq& q^n \sum_{j=1}^{n-1} \left(\frac3{q}\right) ^{j} \\
&\leq& q^n \cdot \frac{\nicefrac{3}{q}}{1-\nicefrac{3}{q}}\\
&\leq& q^{n} \cdot \frac{6}{q},
\end{eqnarray*}
where in the first equality we multiply and divide by $q^n$, in the second equality we reverse the order of summation, in the first inequality we use the bound $$\left(\frac{1}{q}\right)^jj^2\leq \left(\frac{3}{q}\right)^j \Leftrightarrow j^2\leq 3^j$$
which is true for any $j\geq 0$, the second inequality uses the known bound for sum of a geometric series, and in the last inequality we use the bound $1\leq 2(1-\nicefrac{3}{q}),$ which is true for $q\geq 6.$
\end{proof}

The next claim shows that if there are $n$ events, each happens with probability at least $1-\delta$, then at least half of them occur together with probability at least $1-2\delta$. Specifically, for $\delta=0.1$ we get the claim needed in the main text. For ease of notation, for any event $A$, we denote the indicator of $A$ by $I_A.$
\begin{claim}\label{clm:lower_bound_from large_center_to_large_randomness}
Fix $\delta\in(0,1)$. Suppose there are $n$ events $A_1,\ldots, A_n$ such that for every $i\in[n]$, $\Pr(A_i)\geq 1-\delta.$ Then, 
$$\Pr\left(\sum_{i=1}^nI_{A_i} \geq n/2\right)\geq 1-2\delta.$$
\end{claim}
\begin{proof}
Let $\delta\in(0,1)$ and events $A_1,\ldots, A_n$ with $\;\Pr(A_i)\geq 1-\delta,$ for every $i.$ We want to prove that  $$\Pr\left(\sum_{i=1}^nI_{\neg A_i} \geq n/2\right)\leq 2\delta,$$ where $\neg A$ is the complement of $A.$
From the assumption in the claim we know that $$\E\left[\sum_{i=1}^nI_{\neg A_i}\right]\leq\delta n.$$
Thus, from Markov's inequality we have that $$\Pr\left(\sum_{i=1}^nI_{\neg A_i} \geq \frac{1}{2\delta}\cdot\delta n\right)\leq 2\delta.$$
\end{proof}

\begin{claim}\label{clm:norm_of_sum_vectors}
For any $u,v\in \reals^d$ it holds that $$\norm{v+u}^2\leq 2\norm{v}^2 + 2\norm{u}^2.$$
\end{claim}
\begin{proof}
\begin{eqnarray*}
\norm{v+u}^2 
&=& \norm{v}^2+\norm{u}^2 + 2\inner{v}{u}
\leq \norm{v}^2+\norm{u}^2 + 2\norm{v}\norm{u}
\leq 2\norm{v}^2+2\norm{u}^2,
\end{eqnarray*}
where the first inequality follows from Cauchy–Schwarz inequality and the second inequality follows from the  inequality
$0\leq(\norm{u}-\norm{v})^2 = \norm{v}^2+\norm{u}^2 - 2\norm{v}\norm{u}.$
\end{proof}

\subsection*{Proof of Lemma~\ref{lemma:farthest_first_traversal}}
\begin{proof}%[proof of Lemma~\ref{lemma:farthest_first_traversal}]
Take $y=\argmin_{y\in S}\norm{x-y}$ and any $y_1,y_2\in S$ we will show that
$$\norm{y_1-y_2}\geq \norm{x-y}.$$ 
W.l.o.g $y_2$ was added to $S$, after $y_1$ did.
Focus at the time $y_2$ was added to $S$. 
Denote by $l\in S$ the closest point in $S$ at the time to $x.$
Then, since $y_2$ was added to $S$ and not $x$ we know that $$\norm{y_2-y_1}\geq\norm{x-l}\geq\norm{x-y}.$$ 
\end{proof}

\subsection*{Proof of Lemma~\ref{lemma:general_cost_expected_cost}}
\begin{proof}
\begin{eqnarray*}
\E_{j\in[n]}\left[\sum_{i=1}^n\dis{x_i}{x_j}\right] &=& \frac{1}{n}\sum_{i,j\in[n]}\dis{x_i}{x_j}\\
&\leq& \frac Dn\cdot\sum_{i,j\in[n]}\dis{x_i}{\mu}+\dis{\mu}{x_j}\\
&=& 2D\cdot \sum_{i=1}^n\dis{x_i}{\mu},
\end{eqnarray*}
where the inequality follows from Inequality~\ref{eq:generalized_triangle_inequality}.
\end{proof}

% Acknowledgments---Will not appear in anonymized version
% \acks{We thank a bunch of people.}

\end{document}